\documentclass{article}
\usepackage[a4paper, total={6in, 10in}]{geometry}

\usepackage{hyperref}
\usepackage{url}

\usepackage{mylatexstyle}
\usepackage{mathrsfs,bbm}
\usepackage{subfigure}
\usepackage{amsmath}
\usepackage{amssymb}
\usepackage{mathtools}
\usepackage{amsthm}

\usepackage{xcolor}

\title{\huge Improving Implicit Regularization of  SGD with Preconditioning 
for Least Square Problems}

\author{
Junwei Su, Difan Zou$^{\dagger}$, Chuan Wu  \\
Department of Computer Science, University of Hong Kong \\
\texttt{\{jwsu,dzou,cwu\}@cs.hku.hk} 
}

\begin{document}

\maketitle
\begin{abstract}
Stochastic gradient descent (SGD) exhibits strong algorithmic regularization effects in practice and plays an important role in the generalization of modern machine learning. However, prior research has revealed instances where the generalization performance of SGD is worse than ridge regression due to uneven optimization along different dimensions. Preconditioning offers a natural solution to this issue by rebalancing optimization across different directions. Yet, the extent to which preconditioning can enhance the generalization performance of SGD and whether it can bridge the existing gap with ridge regression remains uncertain. In this paper, we study the generalization performance of SGD with preconditioning for the least squared problem. We make a comprehensive comparison between preconditioned SGD and (standard \& preconditioned) ridge regression. Our study makes several key contributions toward understanding and improving SGD with preconditioning. First, we establish excess risk bounds (generalization performance) for preconditioned SGD and ridge regression under an arbitrary preconditions matrix. Second, leveraging the excessive risk characterization of preconditioned SGD and ridge regression, we show that (through construction) there exists a simple preconditioned matrix that can make SGD comparable to (standard \& preconditioned) ridge regression. Finally, we show that our proposed preconditioning matrix is straightforward enough to allow robust estimation from finite samples while maintaining a theoretical improvement. Our empirical results align with our theoretical findings, collectively showcasing the enhanced regularization effect of preconditioned SGD.
\let\thefootnote\relax\footnotetext{$\dagger$: corresponding authors}
\end{abstract}

\section{Introduction}
Stochastic Gradient Descent (SGD) plays a pivotal role in the field of deep learning, due to its simplicity and efficiency in handling large-scale data and complex model architectures. Its significance extends beyond these operational advantages. Research has revealed that SGD may be a key factor enabling deep neural networks, despite being overparameterized, to still achieve remarkable generalization across various machine learning applications~\citep{tian2023recent}. A particularly intriguing aspect of SGD is its inherent ability to provide \emph{implicit regularization}~\citep{zhou2020towards,wu2020direction,pesme2021implicit,dauber2020can,li2022implicit,sekhari2021sgd,azizan2018stochastic,smith2021on}. This means that even in the absence of explicit regularizers, SGD can naturally guide overparameterized models towards solutions that generalize well on unseen data. This implicit regularization characteristic of SGD is a critical area of study, enhancing the understanding of why and how these large machine learning models perform effectively in real-world scenarios.

A growing body of work has emerged, dedicated to understanding and characterizing the generalization performance of SGD in the context of the linear regression problem. These studies have unveiled a close relationship between the implicit regularization inherent to SGD and the explicit $L_2$-type regularization (ridge)~\citep{zhang2017understanding,tikhonov1963solution,ali2019continuous,ali2020implicit,gunasekar2018characterizing,suggala2018connecting,zou2021benefits}. In particular,  ~\citep{zou2021benefits} has presented an instance-based comparison study between the generalization performance (measured by \emph{excessive risk}) of SGD and ridge regression. They show that (single-pass) \emph{SGD can achieve comparable generalization performance than that of ridge regression in certain problem instances but fall behind in others.}\footnote{we provide concrete examples for further illustration in Appendix~\ref{appendix:concrete_example}.}

A notable insight derived from the comparative analysis conducted by \citep{zou2021benefits} is that the underperformance of SGD in certain scenarios can be attributed to an inherent imbalance optimization within the data covariance matrix. This imbalance results in a substantial bias error, primarily associated with the ``relatively small'' eigenvalues of the data covariance matrix.  To address this limitation of SGD, preconditioning emerges as a natural solution. Preconditioning is a commonly used technique to enhance the performance of the optimizer~\citep{benzi2002preconditioning,avron2017faster,gupta2018shampoo,amari2020does,Kang_2023_CVPR,li2017preconditioned,li2016preconditioned}; it involves modifying the update steps of the algorithm by transforming the data space,  offering a mechanism to rebalance the optimization process across different directions. Despite extensive research that has explored the implicit regularization of SGD and its parallels with ridge regression, the realm of preconditioned SGD remains relatively uncharted. Recognizing this gap in the current body of knowledge, our study investigates the implicit regularization of SGD when augmented with preconditioning. We adhere to the standard settings (least square problem) used in the existing literature~\citep{zou2021benefits,zou2021benign,zou2023benign,tsigler2023benign} and aim to address the following central question of our study:
\begin{center}
\emph{Can preconditioning techniques further enhance the generalization performance of SGD and close the existing gap relative to ridge regression?}
\end{center}

In essence, our objective is to investigate whether a thoughtfully designed preconditioning matrix can empower SGD to be consistently comparable with ridge regression in the least square problem. Furthermore, we aim to delve deeper into the dynamic relationship between SGD and ridge regression when both methods incorporate preconditioning. Achieving these goals presents several key challenges. First and foremost, existing results characterizing the excessive risk of SGD and ridge regression predominantly pertain to scenarios without preconditioning. The introduction of preconditioning alters the learning dynamics of SGD and ridge regression, potentially leading to distinct characterizations of their excessive risk. Secondly, as discussed earlier, the effectiveness of preconditioning for SGD is contingent upon its impact on the eigenspectrum of the data covariance matrix.  A naive design of the preconditioning matrix may fail to align with the excessive risk characteristics of SGD, potentially leading to adverse consequences. Lastly, obtaining precise information regarding the eigenspectrum of the covariance matrix can be impractical in real-world settings. To ensure the practical applicability and validity of the preconditioning technique in such scenarios, we must consider a design that remains sufficiently straightforward to facilitate robust estimation.

\paragraph{Our contributions.}  In this paper, we show (through construction) that there exists a preconditioning design that allows the excessive risk of SGD to be consistently comparable with ridge regression. In addition, we establish corresponding analytical tools to tackle the challenges stated earlier. The main contributions of our work are summarized as follows. 

\begin{itemize}[leftmargin=*]
    \item  We extend the existing results that characterize the excessive risk of SGD and ridge regression, as outlined in~\citep{zou2023benign, tsigler2023benign}, to include scenarios incorporating preconditioning. This extension provides valuable insights into how the preconditioning matrix influences the excessive risk associated with both SGD and ridge regression (Theorrem~\ref{theorem:precondition_ridge_fit} \& ~\ref{theorem:precondition_sgd_fit2}).  Importantly, it motivates the design of the preconditioning matrix and enables a comprehensive comparison between SGD and ridge regression within the context of preconditioning.
    \item  Building upon the insights derived from the analysis of excessive risk, we propose a simple (yet effective) preconditioning matrix tailored to SGD, which leverages information from the data covariance matrix. We show that under the theoretical setting (with access to the precise information of data covariance matrix), the excessive risk of SGD with our proposed preconditioning consistently matches that of standard ridge regression and a representative family of preconditioned ridge regression (Theorems~\ref{theorem:HI} \& ~\ref{theorem:HM}). This shows that preconditioning can bridge the existing gap between SGD and ridge regression, supporting the assertion of the central question above. 
    \item  In practical scenarios where precise information about the data covariance matrix is unavailable, our study illustrates that our proposed preconditioning design allows for robust estimation using a finite set of inexpensive unlabeled data. Remarkably, even when we employ the estimated version of our proposed preconditioning matrix, the excessive risk of SGD can still consistently match that of (standard \& preconditioned) ridge regression (Theorems~\ref{theorem:HI_est} \& ~\ref{theorem:HM_est}). This discovery not only highlights the significance of SGD in practical machine learning applications but also underscores the practical advantages offered by preconditioning.
\end{itemize}

\noindent\textbf{Notation.} We use lowercase letters to denote scalars, and boldface letters to denote vectors and matrices, respectively. For a vector $\mathbf{x} $, $\mathbf{x}[i]$ denotes the $i$-th coordinate of $\mathbf{x}$. For two functions $f(x) \ge 0$ and $g(x) \ge 0$ defined on $x > 0$,
we write $f(x) \lesssim g(x)$ if $f(x) \le c\cdot g(x)$ for some absolute constant $c > 0$; 
we write $f(x) \gtrsim g(x)$ if $g(x) \lesssim f(x)$;
we write  $f(x) \eqsim g(x)$ if $f(x) \lesssim g(x)\lesssim f(x)$.
For a vector $\wb \in \RR^d$ and a positive semidefinite matrix $\Hb \in \RR^{d\times d}$, we denote $\norm{\wb}^2_{\Hb} :=  \wb^\top \Hb \wb $.

\section{Related Work}
We provide an overview of the technical advancements in the analysis of ridge regression and SGD, as well as related research comparing SGD to explicit norm-based regularization and ridge regression. Additionally, we discuss 
the current usage of preconditioning in both SGD and ridge regression.

{\bf Excessive Risk Bounds of Ridge Regression. }
Ridge regression holds a central position in the field of machine learning~\citep{hsu2012random, kobak2020optimal, tsigler2023benign, hastie2009elements}. In the underparameterized scenario, there exists comprehensive understanding on the excess risk bounds of ridge regression~\citep{hsu2012random}. However, in the overparameterized regime, a large body of work focuses on characterizing the excess risk of ridge regression in the asymptotic scenario where both the sample size and the dimension approach infinity with a constant ratio~\citep{dobriban2018high, hastie2022surprises, wu2020optimal, xu2019number}. Recent advancements by~\citep{tsigler2023benign} have provided detailed non-asymptotic risk bounds for ordinary least squares and ridge regression in overparameterized contexts. Notably, these bounds depend on the data covariance spectrum and are applicable even when the ridge parameter is zero or negative. This body of work forms a foundation of the theoretical framework for our paper.

{\bf Excessive Risk Bounds of SGD.} In the finite-dimensional setting, risk bounds for one-pass, constant-stepsize SGD have been well-established~\citep{dieuleveut2017harder, bach2013non, jain2017markov, jain2018parallelizing, defossez2015averaged,paquette2022implicit}. In the under-parameterized setting, constant-stepsize SGD with tail-averaging has been recognized for achieving the minimax optimal rate for least squares~\citep{jain2017markov, jain2018parallelizing}. A recent study by~\citep{zou2023benign} expanded on the previous analysis, offering a detailed risk bound for SGD in the overparameterized regime. This work provides closely aligned upper and lower excess risk bounds for constant-stepsize SGD, distinctly defined in relation to the complete eigenspectrum of the population covariance matrix. These findings are critically influential in the context of our paper.

{\bf Implicit Regularization of SGD and Explicit Regularization.}
It is well known that in the least squares problems, multi-pass SGD approaches the solution with the minimum norm, representing a recognized implicit bias of SGD~\citep{neyshabur2014search, zhang2017understanding, gunasekar2018characterizing}. However, this norm-based regularization alone does not fully explain the optimization behaviour of SGD in more extensive scenarios such as convex but non-linear models~\citep{arora2019implicit, dauber2020can, razin2020implicit}. Recent work by~\citep{zou2021benefits} presents a comparative analysis between SGD and ridge regression from the perspective of excessive risk, demonstrating that SGD can match the performance of ridge regression in some scenarios but fall short in others. Their result has motivated us to explore SGD with preconditioning. In this paper, we advance the understanding of SGD by exploring its implicit regularization in the presence of preconditioning.

{\bf Preconditioning in SGD and ridge.} Preconditioning is a widely employed optimization technique~\citep{benzi2002preconditioning} for enhancing the efficiency and effectiveness of optimization algorithms. Its primary objective is to transform complex or ill-conditioned optimization problems into more manageable forms. Preconditioning techniques have been extensively explored in both SGD and ridge regression, serving various purposes~\citep{benzi2002preconditioning,li2017preconditioned,gonen2016solving,amari2020does,xu2023power,Kang_2023_CVPR}. However, the impact of preconditioning on the excessive risk of SGD and ridge regression, and whether it can enable SGD to consistently match the performance of ridge regression, remain open questions that motivate our research.

\section{Problem Setup and Preliminaries}
In this paper, our objective is to close an open question in the existing learning literature. To do so, we explore the usage of precondition in SGD, adhere to the setting in previous studies,  and compare the generalization performance of preconditioned SGD and ridge regression in the context of solving least square problems in the overparameterized regime~\footnote{we present a further discussion on our setting and how our results can be extended in Appendix~\ref{appendix:extension}}. We denote a feature vector in a separable Hilbert space $\cH$ as $\xb\in\cH$. The dimensionality of $\cH$ is denoted as $d$, which can be infinite-dimensional when applicable. We use $y\in\RR$ to denote a response and it is generated as follows:
\begin{align*}
y = \la\xb,\wb^*\ra + \epsilon, 
\end{align*}
where $\wb^* \in \cH$ represents the unknown true model parameter, and $\epsilon \in \RR$ is the model noise. The goal of the least squares problem is to estimate the true parameter $\wb^*$ by minimizing the following:
\begin{align*}
& L(\wb^*) = \min_{\wb\in\cH}L(\wb), \quad \text{where} \quad L(\wb):= \frac{1}{2}\EE_{(\xb,y)\sim \mathbb{D}}\big[(y - \la\wb,\xb\ra)^2\big],
\end{align*}
where $\mathbb{D}$ represents the data distribution. The generalization performance of an estimated $\wb$ found by an algorithm, whether it is SGD or ridge regression, is evaluated based on the {excessive risk}:
$$\cE(\wb):=L(\wb) - L(\wb^*).$$
We use $\Hb := \EE[\xb \xb^\top]$ to denote the second moment of $\xb$, which is the covariance matrix characterizing how the components of $\xb$ vary together. Following the prior literature \citep{bartlett2020benign,zou2021benign}, we made the following assumptions on the data distribution.
\begin{assumption}[Regularity conditions]\label{assump:data_distribution}
$\Hb$ is strictly positive definite and has a finite trace.
\end{assumption}

\begin{assumption}[Bounded Fourth-order Moment]\label{assump:fourth_moment}
 There exists a positive constant $\alpha$ such that for any PSD matrix $\Ab$, it holds that: $\EE[\xb \xb^\top \Ab \xb \xb^\top] - \Hb \Ab \Hb \preceq \alpha \tr(\Ab \Hb )  \cdot  \Hb.$
\end{assumption}

\begin{assumption}[Noise-data Covariance]\label{assump:model_noise} There exists a positive constant $\sigma$ such that:  $\EE_{\xb,\epsilon}[\epsilon^2 \xb \xb^\top] \leq \sigma^2 \Hb.$
\end{assumption}

To analyze $\wb^*$ and $\Hb$ in terms of excessive risk, we introduce the following notation:
\begin{align*}
  {\Hb}_{0:k} := \textstyle{\sum_{i=1}^k}\lambda_i\vb_i\vb_i^\top, \quad
  {\Hb}_{k:\infty} := \textstyle{\sum_{i> k}}\lambda_i\vb_i\vb_i^\top,
\end{align*}
where $\{\lambda_i\}_{i=1}^\infty$ are the eigenvalues of $\Hb$ sorted in
non-increasing order and $\vb_i$'s are the corresponding eigenvectors. 
Then we define:
\begin{align*}
    \|\wb\|^2_{\Hb_{0:k}^{-1}}=
\sum_{ i\le
  k}\frac{(\vb_i^\top\wb)^2}{\lambda_i}, \quad
\|\wb\|^2_{{\Hb}_{k:\infty}}=\sum_{i> k}\lambda_i (\vb_i^\top\wb)^2.
\end{align*}

\noindent\textbf{Preconditioned SGD.} 
We consider single-pass preconditioned SGD with a constant step size and tail-averaging~\citep{bach2013non,jain2017markov,jain2018parallelizing}. At the $t$-th iteration, a fresh example $(\xb_t,y_t)$ is sampled from the data distribution and to perform the update:
\begin{align}\label{eq:pred_sgd}
    {\wb}_{t+1}  = \wb_t - \eta \cdot   {\color{blue}{\Gb}} \nabla l(\wb_t;\xb_t,y_t) 
     = {\wb}_{t} - \eta \cdot (\langle {\wb}_{t},\xb_t \rangle - y_t) \cdot {\color{blue}{\Gb}}  \xb_t
\end{align}
where $\eta>0$ is a constant stepsize (also referred to as the learning rate) and $\Gb$ (highlighted in blue) is the preconditioned matrix --- the key difference from the standard SGD. Note that if $\Gb = \Ib$, then the update rule (\ref{eq:pred_sgd}) reduces to the standard SGD with constant step size. 

After $N$ iterations, which corresponds to the number of samples observed, SGD computes the final estimator as the tail-averaged iterates:
\begin{align}\label{eq:pred_out}
    \wb_{\mathrm{sgd}}(N,\Gb;\eta) := \frac{2}{N}\textstyle{\sum_{t=N/2}^{N-1}}\wb_t.
\end{align}
\noindent\textbf{Preconditioned ridge regression.}
For a given set of $N$ independent and identically distributed samples $\{(\xb_i,y_i)\}_{i=1}^N$, we define $\Xb:= [\xb_1,\dots,\xb_N]^\top\in\RR^{N\times d}$ and $\yb := [y_1,\dots,y_N]^\top\in\RR^d$. Then, preconditioned ridge regression provides an estimator for the true parameter:
\begin{align}\label{eq:predridge_solution}
\wb_{\mathrm{ridge}}(N,\Mb;\lambda) := \arg\min_{\wb\in\cH} \|\Xb\wb - \yb\|_2^2 + \lambda\|\wb\|^2_{\color{blue}{\Mb}},
\end{align}
where $\lambda\ge 0$ is the regularization parameter, and $\Mb$ (highlighted in blue) represents the preconditioned matrix, which distinguishes it from standard ridge regression~\citep{tsigler2023benign}. Note that if $\Mb = \Ib$, the preconditioned ridge solution (\ref{eq:predridge_solution}) also reduces to the standard ridge regression solution. When $\lambda = 0$, the ridge estimator reduces to the ordinary least square estimator~\citep{hastie2009elements}.


\noindent\textbf{General set-up.} In this study, we compare the generalization performance of SGD and ridge regression from the standpoint of excessive risk for the same amount of training samples, which essentially is the notion of Bahadur statistical efficiency~\citep{bahadur1967rates}. Specifically, our goal is to demonstrate the existence of a preconditioning matrix ({through construction}) that enables SGD to achieve a comparable or superior excessive risk to ridge regression {when provided with an equal amount of training data}. For comparison, we adhere to the previous studies and adopt the order-based comparison framework (which focuses on growth rate and is a common framework used in similar theoretical studies). To ensure a meaningful comparison, we focus on a setting where both SGD and ridge regression are considered ``generalizable'', meaning they achieve a vanishing excess risk when $n\rightarrow \infty$ (note that $d$ can also approach $\infty$) with the optimal hyperparameter configurations.  For this purpose, in the forthcoming sections, we assume sufficient training samples (a large enough $N$) and a constant level signal-to-noise ratio, i.e., $\rho^2 = \frac{\|\wb^*\|^2_{\Hb}}{\sigma^2} = \Theta(1).$

\section{Main Results}

\noindent\textbf{Overview of the results.} 
We first present the excessive risk characterization of SGD and ridge with preconditioning. Given our primary objective of demonstrating the comparable performance of SGD over ridge regression, we focus on establishing an excessive risk upper bound for SGD and an excessive risk lower bound for ridge regression. These bounds provide rigorous support for our analysis and claims. Then, building on the insights from the excessive risk analysis, we discuss and present our design for the preconditioning matrix. We then prove that under a theoretical setting (with access to the precise information of the covariance matrix), SGD with our proposed precondition matrix can consistently achieve comparable performance to standard ridge regression. Moreover, under the same setting, we show that SGD with our proposed preconditioning matrix maintains a comparable performance to ridge regression when ridge regression employs a representative family of conditioning. Lastly, we address the practical implications of our findings. In a real-world setting, where access to precise covariance matrix information is unavailable, we show that our preconditioning matrix design allows for robust estimation using finite unlabelled data. Remarkably, SGD, with our proposed preconditioning, still guarantee to maintain a comparable performance over ridge regression in this practical scenario. Detailed proof can be found in the appendix.

\subsection{Excessive Risk of Preconditioned SGD and Ridge Regression} 

\noindent\textbf{Ecessive risk of preconditioned ridge.} We start with characterizing the lower bound of the excessive risk for the preconditioned ridge, which is captured by the following theorem.

\begin{theorem}[Excessive risk lower bound for ridge with preconditioning]\label{theorem:precondition_ridge_fit}    
Consider ridge regression with parameter $\lambda > 0$ and precondition matrix $\Mb$. Suppose Assumptions~\ref{assump:data_distribution}, ~\ref{assump:fourth_moment} and ~\ref{assump:model_noise} hold. Let 
$\hat{\Hb} = \Hb^{1/2}\Mb^{-1}\Hb^{1/2}, \quad \text{and} \quad \hat{\wb}^*=\Mb^{1/2} \wb^*.$
For a constant $\hat{\lambda}$ depending on $\lambda$, $\hat{\Hb}$, $N$ , and $k^* = \min\{k: N\hlambda_k \lesssim \hat{\lambda}\}$, ridge regression has the following excessive risk lower bound:
\begin{align}
    & \mathrm{RidgeRisk}  \gtrsim \underbrace{\frac{\hat{\lambda}^2}{ N^2}\cdot\big\|\hat{\wb}^*\big\|_{\hat{\Hb}_{0:k^*}^{-1}}^2 + \|\hat{\wb}^*\big\|_{\hat{\Hb}_{k^*:\infty}}^2}_{\ridgebias} +   \underbrace{\sigma^2\cdot\bigg(\frac{k^*}{N}+\frac{N}{\hat{\lambda}^2}\sum_{i>k^*}\hat{\lambda}_i^2\bigg)}_{\ridgevar} ,
\end{align}
where $\hlambda_1,\dots, \hlambda_d$ are the sorted eigenvalues for $\hH$ in descending order.
\end{theorem}
Theorem~\ref{theorem:precondition_ridge_fit} provides a sharp characterization of the excessive risk lower bound for the preconditioned ridge. It shows that the excessive risk lower bound of the preconditioned ridge regression depends on the eigenspace of the transformed covariance matrix $\hat{\Hb}$ and $\hat{\wb}^*$. A couple of things should be noted from the result above. First, if we let $\Mb = \Ib$, the result above can be reduced to the lower bound of the excessive risk of standard ridge regression. Therefore, Theorem~\ref{theorem:precondition_ridge_fit} is a more general description of the excessive risk lower bound of ridge regression~\citep{tsigler2023benign}. Second, the excessive risk bound is composed of two components, bias and variance. Third, the leading ($i \leq k^*$) eigenvectors and the tail ($i >k^*$) have different effects on both the bias and variance. Lastly, $k^*$ is a constant depending on the problem instance. These observations play an important role in our analysis later. 

\noindent\textbf{Excessive risk of preconditioned SGD.}  Next, we characterize the excessive risk upper bound for the preconditioned SGD, which is captured by the following theorem.

\begin{theorem}[Excessive risk upper bound of preconditioned SGD]\label{theorem:precondition_sgd_fit2}    
Let $\Gb$ be a given preconditioned matrix. Consider SGD update rules with (\ref{eq:pred_sgd}) and (\ref{eq:pred_out}). Suppose Assumptions~\ref{assump:data_distribution}, ~\ref{assump:fourth_moment} and ~\ref{assump:model_noise} hold. Let 
$\tilde{\Hb} = \Gb^{1/2}\Hb\Gb^{1/2}, \quad \text{and} \quad \tilde{\wb}^*=\Gb^{-1/2} \wb^*.$ Suppose the step size $\eta$ satisfies $\eta \leq 1/ \tr (\tilde{\Hb})$. Then SGD has the following excessive risk upper bound for arbitrary $k_1, k_2 \in [d]$:
\begin{align}\label{eq:predsgd_risk}
    & \mathrm{SGDRisk} \lesssim  \underbrace{\frac{1}{\eta^2 N^2}\cdot\big\|e^{-N\eta \tilde{\Hb}}\tilde{\wb}^*\big\|_{\tilde{\Hb}_{0:k_1}^{-1}}^2 + \|\tilde{\wb}^*\big\|_{\tilde{\Hb}_{k_1:\infty}}^2}_{\sgdbias}
   + \underbrace{\big(\sigma^2+ \|\tw^*\|_{\tH}\big) \cdot \bigg(\frac{k_2}{N}+N\eta^2\sum_{i>k_2}\tilde{\lambda}_i^2\bigg)}_{\sgdvar},
\end{align}
where $\tlambda_1,\dots, \tlambda_d$ are the sorted eigenvalues of $\tH$ in descending order.
\end{theorem}

Theorem~\ref{theorem:precondition_sgd_fit2} provides a sharp characterization of the excessive risk upper bound for the preconditioned SGD. It shows that the excessive risk upper bound of the preconditioned SGD depends on the eigenspace of the transformed covariance matrix $\tH$ and $\tw^*$. A couple of points should be noted from the result. First, if we let $\Gb = \Ib$, then the excessive risk upper bound also reduces to the excessive risk upper bound for standard SGD. Therefore, Theorem~\ref{theorem:precondition_sgd_fit2} expands on the standard result~\citep{zou2023benign}, providing a more general characterization of the excessive risk of SGD under preconditioning. Second, similar to the lower bound of excessive risk for the ridge, the excessive risk upper bound of SGD can also be decomposed into bias and variance. In addition, the leading eigenvectors and the tail eigenvectors (determined by $k_1$ and $k_2$) also have different effects on the bias and variance. Lastly, different from the ridge regression, the result above for SGD holds for arbitrary $k_1, k_2 \in [d]$. This gives a degree of freedom that we can adjust when comparing with ridge regression.

\subsection{Design of The Precondition Matrix}
\noindent\textbf{Motivation of design.} Before proceeding, we highlight some critical observations from the excessive risk analysis above and ~\citep{zou2021benefits} that guide our choice for the preconditioning matrix:
\begin{enumerate}[leftmargin=*, nosep]
    \item Theorem~\ref{theorem:precondition_ridge_fit} represents an excessive risk lower bound for ridge and Theorem~\ref{theorem:precondition_sgd_fit2} represents an excessive risk upper bound for SGD. 
    Furthermore, as both excessive risk bound can be decomposed into bias and variance, to prove that SGD is comparable with ridge regression, it is sufficient to show: 
    \begin{align*}
        \mathrm{SGDBias}  \lesssim \mathrm{RidgeBias}, \quad \text{and} \quad 
         \mathrm{SGDVariance}  \lesssim \mathrm{RidgeVariance}.
    \end{align*}
    \item The excessive risk result for the reconditioned SGD holds for arbitrary $k_1,k_2 \in [d]$. Setting $k_1,k_2 = k^*$ in the excessive risk upper bound of the preconditioned SGD, we can see that the variance of the SGD and the tail ($i > k^*$) of the bias match closely to those in the ridge.
    \item The bias induced by the leading eigenvectors in the excessive risk of SGD exhibits exponential decay, a primary factor allowing for bias control in preconditioned SGD. Therefore, it is advantageous for the excessive risk of SGD to scale the data spectrum toward the leading eigenvectors. 
    \item Previous research ~\citep{zou2021benefits} has shown that there are instances where for SGD to achieve a comparable excessive risk with ridge regression, it requires a sample size of,
    \begin{align*}
        N_{\mathrm{sgd}} \geq \Theta(\kappa (N_{\mathrm{ridge}})) N_{\mathrm{ridge}},
    \quad \text{where} \quad  \kappa(N_{\mathrm{ridge}}) = \frac{\tr(\Hb)}{N_{\mathrm{ridge}} \lambda_{\min\{N_{\mathrm{ridge}}, d\}}},
    \end{align*}
    and $N_{\mathrm{sgd}}$ and $N_{\mathrm{ridge}}$ are the sample size for SGD and ridge respectively. We can observe that the gap between SGD and ridge is closely related to the quantity $\kappa(N_{\mathrm{ridge}})$, which reflects the flatness of the eigenspectrum of $\Hb$ in the leading subspace.
\end{enumerate}

\noindent\textbf{Design of precondition.} Building on the above observations and insights,  we aim to design a preconditioning matrix $\Gb$ that can amplify and flatten the relative signal strength for the leading eigenspace. To achieve this, we seek a design for $\Gb$ that deviates from isotropic transformations and is tailored to the signal strengths of the eigenspectrum. Consequently, we propose the following formulation for the preconditioning matrix:
\begin{align} \label{eq:precondition}
    \Gb = (\beta \Hb + \Ib)^{-1}, 
\end{align}
where $\beta \geq 0$ is a user-define (tunable) constant. To denote the influence of this parameter, we use ${\wb}_{\mathrm{sgd}}(N,\Gb; \beta, \eta )$ to indicate the tunable parameter $\beta$.  It can be observed that as $\beta$ approaches zero, $\Gb$ converges to the identity matrix $\Ib$, effectively restoring preconditioned SGD to the standard SGD. Conversely, as we increase $\beta$, the relative effect of $\Ib$ diminishes for the leading eigenspectrum because $\beta \Hb$ becomes the dominating factor. Consequently, by tuning $\beta$ appropriately, we can increase and flatten the relative signal strength of the leading eigenspace, leveraging the exponential bias decay from SGD and improving $\kappa(N{\mathrm{ridge}})$.


\subsection{Effectiveness of Preconditioned SGD in Theoretical Setting}
We start by assuming that the precise information about $\Hb$ is given. Under this setting, we investigate two cases: {\bf Case I}: the comparison between preconditioned SGD and standard ridge regression ($\Mb = \Ib$); and {\bf Case II}: the comparison between preconditioned SGD and ridge regression with a representative family of preconditioning matrices.

The following theorem demonstrates that SGD with our proposed preconditioned matrix can indeed be comparable to standard ridge regression, i.e., achieving comparable or lower excessive risk.
\begin{theorem}[Preconditioned-SGD comparable to standard ridge regression]\label{theorem:HI} Consider a preconditioned matrix for SGD as,  $\Gb = (\beta \Hb + \Ib)^{-1}.$
Let $N$ be a sufficiently large training sample size for both ridge and SGD.  Then, for any ridge regression solution that is generalizable and any $\lambda > 0$, there exists a choice of stepsize $\eta^*$ and a choice of $\beta^*$ for SGD such that
$\cE\big[{\wb}_{\mathrm{sgd}}(N,\Gb; \beta^*, \eta^*)\big]  \lesssim  \cE\big[\wb_{\mathrm{ridge}}(N,\Ib;\lambda)\big].$
\end{theorem}

Theorem~\ref{theorem:HI} illustrates that preconditioned SGD can indeed be comparable to ridge regression, which expands the theoretical understanding of the implicit regularization of SGD and affirmingly supports our central research questions regarding the improvement of SGD's implicit regularization through preconditioning. It demonstrates that our proposed preconditioning matrix design effectively increases the relative strength of the leading eigenspace and improves $\kappa(N_{\mathrm{ridge}})$, thereby closing the existing gap between SGD and ridge regression as observed in previous work~\citep{zou2021benefits}.


Next, we explore the implications when ridge regression is also equipped with a precondition. To make the analysis feasible, we focus on a representative family of preconditioning matrices that operate on the eigenspectrum of the data covariance matrix~\citep{woodruff2014sketching,gonen2016solving,clarkson2017low}. Consequently, we assume that the preconditioning matrix $\Mb$ for ridge shares the same eigenbasis as the data covariance matrix $\Hb$, preserves the order of eigenvalues and is strictly positive. 

\begin{theorem}[Preconditioned-SGD comparable to preconditioned ridge regression]\label{theorem:HM} 
Let $\Mb$ be the preconditioned matrix for ridge.
Consider a preconditioned matrix for SGD as $\Gb = \Mb^{-1/2}(\beta \hH + \Ib)^{-1}\Mb^{-1/2},$
where $\hH = \Hb^{1/2}\Mb^{-1}\Hb^{1/2}.$
Let $N$ be a sufficiently large training sample size for both ridge and SGD.  Then, for any ridge regression solution that is generalizable and any $\lambda > 0$, there exists a choice of stepsize $\eta^*$ and a choice of $\beta^*$ for SGD such that $\cE\big[{\wb}_{\mathrm{sgd}}(N,\Gb; \beta^*, \eta^*)\big]  \lesssim  \cE\big[\wb_{\mathrm{ridge}}(N,\Mb;\lambda)\big].$
\end{theorem}
Theorem~\ref{theorem:HM} demonstrates that, for a representative family of preconditioning matrices used in ridge regression, SGD with our proposed precondition matrix can still achieve comparable performance with ridge regression. Specifically, for a given precondition matrix $\Mb$ applied to ridge regression, we can find and adjust $\Gb$ for preconditioned SGD in a way that matches the performance of preconditioned ridge regression. Moreover, by substituting the choice of $\Gb$ from Theorem~\ref{theorem:HI} and \ref{theorem:HM} into $\tH = \Gb^{1/2} \Hb \Gb^{1/2}$, it becomes evident that both methods converge to the same canonical form with respect to $\Hb$ and $\hH$. Consequently, the influence of the design of $\Gb$ in Theorem~\ref{theorem:HM} can be reduced to the case presented in Theorem~\ref{theorem:HI}, with an updated data covariance matrix $\hH$. Additionally, one might naturally question whether the reverse direction holds true with preconditioning, i.e., whether ridge regression can similarly be guaranteed to match the performance of preconditioned SGD through appropriate preconditioning and tuning. We demonstrate in Appendix~\ref{appendix:ridge_match_sgd} (Theorem~\ref{theorem:MG}) that this is indeed the case, suggesting an inherent correlation between SGD and ridge regression.



  

\begin{figure*}[t!]
\subfigure[$\lambda_i=i^{-1}, \wb^{*}{[i]}=1$]{
\includegraphics[width=.31\textwidth]{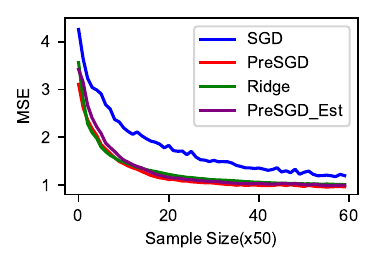}
}
\subfigure[$\lambda_i=i^{-1}, \wb^{*}{[i]}=i^{-1}$]{
\includegraphics[width=.31\textwidth]{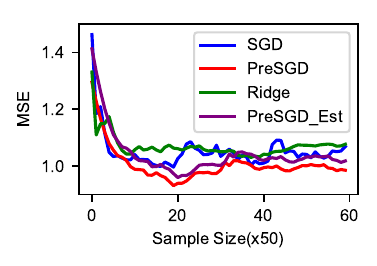}
}
\subfigure[$\lambda_i=i^{-1}, \wb^{*}{[i]}=i^{-10}$]{
\includegraphics[width=.31\textwidth]{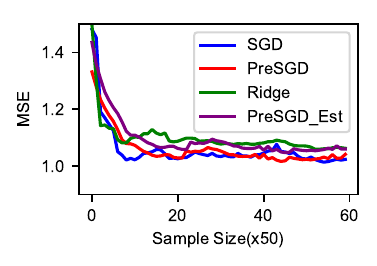}
}
\\
\subfigure[$\lambda_i=i^{-2}, \wb^{*}{[i]}=1$]{
\includegraphics[width=.31\textwidth]{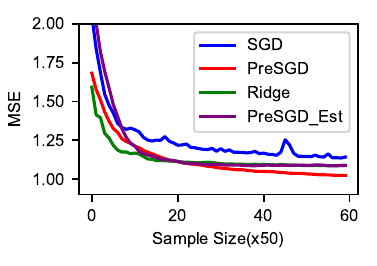}
}
\subfigure[$\lambda_i=i^{-2}, \wb^{*}{[i]}=i^{-1}$]{
\includegraphics[width=.31\textwidth]{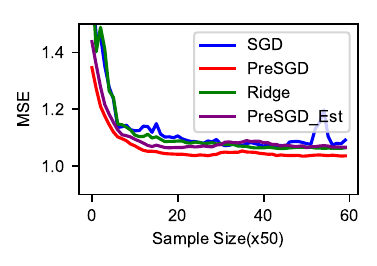}
}
\subfigure[$\lambda_i=i^{-2}, \wb^{*}{[i]}=i^{-10}$]{
\includegraphics[width=.31\textwidth]{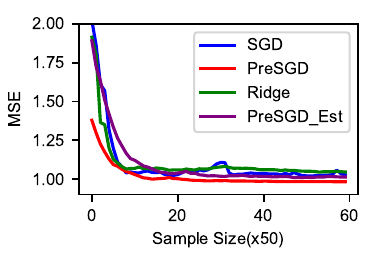}
}
\caption{Generalization performance comparison among SGD, Ridge regression, PreSGD and PreSGD-Est, where the stepsize $\eta$, controlling factor $\beta$ and regularization parameter $\lambda$ are fine-tuned to achieve the best performance. The problem dimension is $d=200$ and the variance of model noise is $\sigma^2=1$. We consider $6$ combinations of $2$ different covariance matrices and $3$ different ground truth model vectors. The generalization performance (excessive risk) is measured using mean squared error (MSE) on an independent test set of size $1000$ that is generated from the same problem instance. A set of unlabelled data of the same size as the training set is drawn from the same problem instance to estimate $\tG$ in (PreSGD-Est). The plots are averaged over $10$ independent runs.}
\label{fig:generalization_perf}
\end{figure*}
 
\subsection{Effectiveness of Preconditioned-SGD with Estimation}
In the previous sections, we assumed access to the precise information about the data covariance matrix $\Hb$. This section addresses the more practical scenarios where obtaining the exact data covariance matrix $\Hb$ is infeasible. In such cases, we consider situations where we have access to a set of inexpensive unlabeled data, denoted as $\{\tilde{\xb}_i\}_{i=1}^M$, drawn from the same distribution as the training data. Leveraging this set of unlabeled data, we first estimate $\Hb$ with the unlabeled data and subsequently design the preconditioned matrix using the estimated covariance information. We denote the estimated covariance matrix as $\mathbf{\Sigma}=M^{-1}\tilde{\Xb}\tilde{\Xb}^{\top}$,
where $\tilde{\Xb} = [\tilde{\xb}_1,...,\tilde{\xb}_M] \in \mathbb{R}^{d \times M}$ is a set of $M$ unlabeled data, which is often accessible and cost-effective to obtain in practice. Moreover, in the case where these data are unavailable, we can use half of the training data to approximate $\bSigma$ to decouple the randomness in theory. In practice, if using training data to perform both stochastic gradient updates and precondition estimation, the preconditioned SGD would resemble the well-known Newton methods to some extent~\footnote {we provide a more in-depth comparison between our method and Newton's method in Appendix~\ref{appendix:newton}}.

The following results demonstrate that our proposed precondition matrix allows for robust estimation from unlabeled data and still allows preconditioned SGD to maintain a theoretical improvement and remain comparable to ridge regression.
\begin{theorem}[Preconditioned-SGD with estimated $\tG$ comparable to standard ridge regression]\label{theorem:HI_est}
Consider a preconditioned matrix for SGD as $\tG = (\beta  \mathbf{\Sigma}  + \Ib)^{-1}.$
Let $N$ be a sufficiently large training sample size for both ridge and SGD.  Then, for any ridge regression solution that is generalizable and any $\lambda > 0$, there exist a choice of stepsize $\eta^*$ and a choice of $\beta^*$ for SGD such that with probability at least $1-\exp(-\Omega(M))$, we have $\cE\big[{\wb}_{\mathrm{sgd}}(N,\tG; \beta^*, \eta^*)\big]  \lesssim  \cE\big[\wb_{\mathrm{ridge}}(N,\Ib;\lambda)\big].$
\end{theorem}

\begin{theorem}[Preconditioned-SGD with estimated $\tG$ comparable to preconditioned ridge regression]\label{theorem:HM_est} 
Let $\Mb$ be the preconditioned matrix for ridge. Consider a preconditioned matrix for SGD as
$\tG = \Mb^{-1/2}(\beta \hH  + \Ib)^{-1}\Mb^{-1/2},$
where $\hH = \mathbf{\Sigma}^{1/2}\Mb^{-1}\mathbf{\Sigma}^{1/2}.$
Let $N$ be a sufficiently large training sample size for both ridge and SGD.  Then, for any ridge regression solution that is generalizable and any $\lambda > 0$, there exists a choice of stepsize $\eta^*$ and a choice of $\beta^*$ for SGD such that with probability at least $1-\exp(-\Omega(M))$, we have $\cE\big[{\wb}_{\mathrm{sgd}}(N,\tG; \beta^*, \eta^*)\big]  \lesssim  \cE\big[\wb_{\mathrm{ridge}}(N,\Mb;\lambda)\big].$
\end{theorem}

Theorems~\ref{theorem:HI_est} and \ref{theorem:HM_est} are variations of Theorems \ref{theorem:HI} and ~\ref{theorem:HM} that incorporate empirical estimation. They establish the efficacy of our proposed preconditioned matrix design and demonstrate that even when the data covariance matrix $\Hb$ is estimated from unlabeled data, SGD with our proposed preconditioned matrix can still be comparable to (standard \& preconditioned) ridge regression with a theoretical guarantee. The introduction of randomness due to empirical estimation is reflected in the probabilistic bound related to the estimation sample size $M$ in the theorem statements. Nonetheless, despite the similarities in their formal expressions, establishing Theorems \ref{theorem:HI_est} and \ref{theorem:HM_est} as opposed to Theorems \ref{theorem:HI} and ~\ref{theorem:HM} presents distinctive technical challenges. Our analytical approach for comparing SGD and ridge relies on the characterization of excessive risk within the eigenspace. The randomness induced by the empirical estimation can introduce stochastic fluctuations to the eigenvalues within the eigenspace, thereby complicating the analysis.  Therefore, to properly analyze and compare SGD and ridge under conditions involving empirical estimation, it is imperative to employ a specialized set of technical tools designed to mitigate the impact of the randomness introduced through empirical estimation.



\subsection{Empirical Study}
We perform experiments on Gaussian least square problem\footnote{Anonymous link to the implementation: https://github.com/jwsu825/presgd}.
We consider $6$ problem instances, which are the combinations of $2$ different covariance matrices $\Hb$, $\lambda_i=i^{-1}$ and $\lambda_i=i^{-2}$, and $3$ different ground-truth vectors $\wb^*$, $\wb^*[i]=1$, $\wb^*[i]=i^{-1}$, and $\wb^*[i] = i^{-10}$. This empirical study aims to answer the following key questions:
\begin{enumerate}[leftmargin=*, nosep]
    \item Can preconditioning indeed improve the generalization performance of SGD? 
    \item Can SGD with preconditioning indeed be consistently comparable with ridge?
    \item Is preconditioning still beneficial when the preconditioning matrix is empirically estimated?
\end{enumerate}
We compare the generalization performance of four optimization methods: SGD, Ridge Regression, Preconditioned-SGD (PreSGD), and PreSGD with estimation (PreSGD-Est) across the 6 problem instances. A set of unlabelled data of the same size as the training set is drawn from the same problem instance to estimate $\tG$ in PreSGD-Est. The results of these comparisons are given in Figure~\ref{fig:generalization_perf}. In cases where Ridge Regression outperforms SGD, such as when $\wb^*[i]=1$, both PreSGD and PreSGD-Est demonstrate the ability to achieve comparable performance to Ridge Regression. Note that when $\wb^*[i]=1$, the bias error will be large. Therefore the improvement of PreSGD and PreSGD-Est in this case well supports our intuition and theory regarding the design of preconditioning. Conversely, in scenarios where SGD exhibits comparable or superior performance to ridge regression, such as when $\wb^*[i]=i^{-1}$ or $\wb^*[i]=i^{-10}$,  both PreSGD and PreSGD-Est manage to further enhance the generalization performance. These empirical findings support the hypotheses posed in our theoretical analysis, validating the effectiveness of preconditioning in improving the generalization performance of SGD. 

\section{Discussion}
{\bf Conclusion.} We conduct an in-depth exploration of implicit regularization in SGD with preconditioning. Our analysis extends the boundaries of existing understanding by characterizing the excessive risk associated with both SGD and ridge regression within the context of preconditioning. We show that preconditioned SGD can be consistently comparable with ridge regression with and without preconditioning, effectively closing the performance gap between SGD and ridge. Furthermore, we illustrate that our proposed preconditioned matrix can be robustly estimated using readily available, inexpensive, and unlabeled data, affirming its practical feasibility and real-world applicability. This research underscores the vital role of SGD in machine learning and emphasizes its potential for further improvement through the utilization of preconditioning techniques. 

{\bf Limitation and Future Work.} Our primary objective in this study was to address an open question from the literature, and we consequently focused on the linear regression problem to maintain consistency with the setting of the previous study. Looking ahead, a natural progression for future research involves expanding our investigation beyond linear models and SGD. We aim to explore whether preconditioning can also enhance the implicit regularization of SGD or similar algorithms within the context of other linear and potentially nonlinear models. More specifically, we are considering the following areas for future research: 1) extending the analysis to explore preconditioned SGD for locally quadratic models, where decaying step sizes are essential for reducing risk to zero; 2) examining the application of preconditioning to SGD in the context of matrix regression~\citep{wu2023many}; and 3) analyzing the effects of preconditioning on SGD and GLM-tron~\citep{kakade2011efficient} in ReLU regression (we have presented preliminary results and discussions for this direction in Appendix~\ref{appendix:relu}).

\bibliography{reference}
\bibliographystyle{iclr2024_conference}

\appendix
\section{Excessive Risk Analysis of Preconditioned SGD and Ridge }\label{appendix:sgd_excessive_risk}

In this part, we will mainly follow the proof technique and results in ~\citep{zou2021benign} and ~\citep{tsigler2023benign} that is developed to sharply characterize the excess risk bound for SGD (with tail-averaging) and ridge. Here, we extend their proof into the cases with precondition. 

\subsection{Ecessive Risk of SGD with precondition}
First, we introduce/recall some notations and definitions that will be repeatedly used in the subsequent analysis.  Recall that $\Hb=\EE[\xb\xb^\top]$ be the covariance of data distribution and $\tilde{\Hb} = \Gb^{1/2}\Hb\Gb^{1/2}$ as the preconditioned covariance. It is easy to verify that $\Hb$ is a diagonal matrix with eigenvalues $\lambda_1,\dots,\lambda_d$. Here, we slightly abuse the notation and use $\Hb[i]$ to indicate the $i$-th eigenvalues corresponding to the $i$-th leading eigenvector.
Let $\wb_t$ be the $t$-th iterate of the preconditioned SGD.
Then, we define $\betab_t:=\wb_t-\wb^*$ as the centred preconditioned SGD iterate.
 Then we define $\betab_t^{\bias}$ and $\betab_t^{\var}$ as the bias error and variance error respectively, which are described by the following update rule:
\begin{align}\label{eq:update_rule_eta_t_pred}
\betab_t^{\bias} &= \big(\Ib-\eta \Gb \xb_t\xb_t^\top\big)\betab_{t-1}^{\bias}, && \betab_0^{\bias}=\betab_0,\notag\\
\betab_t^{\var} &= \big(\Ib-\eta\Gb\xb_t\xb_t^\top\big)\betab_{t-1}^{\var} + \eta\varepsilon\Gb\xb_t,&&\betab_0^{\var}=\boldsymbol{0}. \notag
\end{align}

When $\Gb = \Ib$, the preconditioned SGD reduces to the standard SGD iterate and excessive risk in this case be characterized by the following theorem.

\begin{theorem}[Extension of Theorem 5.1 in \cite{zou2021benign}]\label{thm:standard_sgd}
    Consider SGD with tail-averaging with initialization $\wb_0=\bm{0}$.
Suppose the stepsize satisfies $\eta\lesssim 1/\tr(\Hb)$.
Then the excess risk can be upper-bounded as follows,
\begin{align*}
\EE [L(\wb_{\mathrm{sgd}}(N;\eta))] - L(\wb^*)
&\le  \sgdbias + \sgdvar,
\end{align*}
where 
\begin{align*}
 \sgdbias & \lesssim \frac{1}{\eta^2N^2}\cdot\big\|\exp\big(-N\eta \Hb \big)\wb^*\big\|_{\Hb_{0:k_1}^{-1}}^2 + \big\|\wb^*\big\|_{\Hb_{k_1:\infty}}^2 \\
 \sgdvar & \lesssim \frac{\sigma^2 + \|\wb^*\|_\Hb^2}{N}\cdot\bigg(k_2 + N^2\eta^2 \sum_{i> k_2}\lambda_i^2\bigg).
\end{align*}
where $k_1,k_2\in[d]$ are arbitrary.
\end{theorem}

\begin{proof}[Proof of Theorem~\ref{theorem:precondition_sgd_fit2}]
Using the fact that $$ \yb_t = \langle \wb^*, \xb_t \rangle + \varepsilon, $$
and the update rules of preconditioned SGD,
$$  {\wb}^{(t+1)} =   {\wb}^{(t)} - \eta \cdot (\langle {\wb}^{(t)},\xb_t \rangle - y_t) \cdot {\Gb}  \xb_t$$
we can further obtain
\begin{align*}
    \wb^{(t+1)} - \wb^* = \wb^{(t)} - \wb^* - \eta \Gb \xb_t \xb_t^\top (\wb^{(t)} - \wb^*) - \eta \Gb \xb_t \cdot \varepsilon\\
= (\Ib - \eta \Gb \xb_t \xb_t^\top)(\wb^{(t)} - \wb^*) - \eta \Gb \xb_t \cdot \varepsilon.
\end{align*}
Then multiplying by $ \Gb^{-\frac{1}{2}} $ on both sides, we can obtain,
\begin{align*}
\Gb^{-\frac{1}{2}}(\wb^{(t+1)} - \wb^*) &= \Gb^{-\frac{1}{2}}(\wb^{(t)} - \wb^*) - \eta \Gb^{\frac{1}{2}} \xb_t \xb_t^\top (\wb^{(t)} - \wb^*) - \eta \Gb^{\frac{1}{2}} \xb_t \cdot \varepsilon\\
&= (\Ib - \eta \Gb^{\frac{1}{2}} \xb_t \xb_t^\top \Gb^{\frac{1}{2}}) \Gb^{-\frac{1}{2}}(\wb^{(t)} - \wb^*) - \eta \Gb^{\frac{1}{2}} \xb_t \cdot \varepsilon.
\end{align*}

Then we consider the preconditioned error covariance
\begin{align*}
    \mathrm{{\bf Err}}_t &= \mathbb{E}\left[\Gb^{-\frac{1}{2}}(\wb^{(t+1)} - \wb^*) \otimes \Gb^{-\frac{1}{2}}(\wb^{(t+1)} - \wb^*)\right] \\
    & = \mathbb{E}\left[\Gb^{-\frac{1}{2}}(\wb^{(t)} - \wb^*)(\wb^{(t)} - \wb^*)^\top \Gb^{-\frac{1}{2}}\right],
\end{align*}

where the expectation is taken with respect to the randomness of the SGD algorithm. Then we can get the following update form for this error covariance:
\begin{align*}
    \mathrm{{\bf Err}}_{t+1} &= \mathbb{E}\left[(\Ib - \eta \Gb^{\frac{1}{2}}\xb_t \xb_t^\top \Gb^{\frac{1}{2}})\mathrm{{\bf Err}}_t (\Ib - \eta \Gb^{\frac{1}{2}}\xb_t \xb_t^\top \Gb^{\frac{1}{2}})\right] + \eta^2 \mathbb{E}\left[\varepsilon^2 \cdot \Gb^{\frac{1}{2}}\xb_t \xb_t^\top \Gb^{\frac{1}{2}}\right] \\
& = (\Ib - \eta \Gb^{\frac{1}{2}}\Hb\Gb^{\frac{1}{2}})\mathrm{{\bf Err}}_t(\Ib - \eta \Gb^{\frac{1}{2}}\Hb\Gb^{\frac{1}{2}}) + \eta^2 \mathbb{E}_{\xb_t,\varepsilon}\left[\varepsilon^2 \cdot \Gb^{\frac{1}{2}}\xb_t \xb_t^\top \Gb^{\frac{1}{2}}\right] \\
& +\eta^2 \left(\mathbb{E}\left[\Gb^{\frac{1}{2}}\xb_t \xb_t^\top \Gb^{\frac{1}{2}}\mathrm{{\bf Err}}_t \Gb^{\frac{1}{2}}\xb_t \xb_t^\top \Gb^{\frac{1}{2}}\right] - \Gb^{\frac{1}{2}}\Hb\Gb^{\frac{1}{2}}\mathrm{{\bf Err}}_t \Gb^{\frac{1}{2}}\Hb\Gb^{\frac{1}{2}}\right).
\end{align*}

Then by Assumptions~\ref{assump:model_noise} and ~\ref{assump:fourth_moment}, we can get
\begin{align*}
    \mathrm{{\bf Err}}_{t+1} & \preceq (\Ib - \eta \Gb^{\frac{1}{2}}\Hb\Gb^{\frac{1}{2}})\mathrm{{\bf Err}}_t(\Ib - \eta \Gb^{\frac{1}{2}}\Hb\Gb^{\frac{1}{2}}) \\
& + \eta^2 \cdot \sigma^2 \cdot \Gb^{\frac{1}{2}}\Hb\Gb^{\frac{1}{2}} + \eta^2 \cdot \tr(\Gb^{\frac{1}{2}}\Hb\Gb^{\frac{1}{2}}\mathrm{{\bf Err}}_t) \cdot \Gb^{\frac{1}{2}}\Hb\Gb^{\frac{1}{2}}.
\end{align*}
Besides, we can also verify that
\begin{align*}
    \Delta(\wb_t) &= \langle(\wb_t - \wb^*)(\wb_t - \wb^*), \Hb\rangle\\ 
 &= \langle \Gb^{-\frac{1}{2}}(\wb_t - \wb^*)(\wb_t - \wb^*)\Gb^{-\frac{1}{2}}, \bG^{\frac{1}{2}}\Hb\Gb^{\frac{1}{2}}\rangle  \\
 &= \langle \mathrm{{\bf Err}}_t, \Gb^{\frac{1}{2}}\Hb\Gb^{\frac{1}{2}}\rangle.
\end{align*}

Therefore, the dynamics of the preconditioned SGD can be characterized by studying the dynamics of standard SGD, using a transformed data covariance matrix $\tH = \Gb^{\frac{1}{2}}\Hb\Gb^{\frac{1}{2}}$, ground truth vector $ \tw^* = \Gb^{-\frac{1}{2}}\wb^* $, and iterate $ \Gb^{-\frac{1}{2}}\wb_{t}$. Therefore, following the proofs of Theorem 5.1 in ~\citep{zou2021benign} with the above modification, we can immediately obtain that for arbitrary $k_1,k_2 \in [d]$, we have that,
\begin{align*}
            \mathrm{SGDRisk}
&\lesssim \underbrace{\frac{1}{\eta^2 N^2}\cdot\big\|\exp(-N\eta \tilde{\Hb})\tilde{\wb}^*\big\|_{\tilde{\Hb}_{0:k_1}^{-1}}^2 + \|\tilde{\wb}^*\big\|_{\tilde{\Hb}_{k_1:\infty}}^2}_{\mathrm{SGDBias}} \\
&\qquad + \underbrace{\big(\sigma^2+ \|\tw^*\|_{\tH}\big) \cdot \bigg(\frac{k_2}{N}+N\eta^2\sum_{i>k_2}\tilde{\lambda}_i^2\bigg)}_{\mathrm{SGDVariance}}.
\end{align*}
\end{proof}

\subsection{Ecessive Risk of ridge regression with precondition}
First, recall that the standard ridge regression is equivalent to the following least square problem,
\begin{align*}
 \arg\min_{{\wb}} \|{{\Xb}{\wb}} - \yb\|_2^2 + \lambda\|{{\wb}}\|_{2}^2,
\end{align*}
Then, we have the following extension of Lemmas 2 \& 3 in \citep{tsigler2023benign} for characterizing the excessive risk of ridge regression.
\begin{theorem}[Theorem B.2 in ~\citep{zou2021benefits}]\label{thm:lowerbound_ridge}
 Let $\lambda\ge 0$ be the regularization parameter, $n$ be the training sample size and $\wb_{\mathrm{ridge}}(N;\lambda)$ be the output of ridge regression. 
 Then 
\begin{align*}
\EE\big[L(\wb_{\mathrm{ridge}}(N;\lambda))\big]-L(\wb^*) = \ridgebias +  \ridgevar ,
\end{align*}
and there is some absolute constant $b > 1$, such that for
\[k^* := \min\left\{k: b  \lambda_{k+1} \le  \frac{\lambda+\sum_{i>k}\lambda_i}{n }\right\}, \]
the following holds:
\begin{align*}
\ridgebias &\gtrsim\bigg(\frac{\lambda+\sum_{i>\kr}\lambda_i}{N}\bigg)^2\cdot\|\wb^*\|_{\Hb_{0:\kr}^{-1}}^2+\|\wb^*\|_{\Hb_{\kr:\infty}}^2,\notag\\
\ridgevar &\gtrsim\sigma^2\cdot\bigg\{\frac{\kr}{N}+\frac{N\sum_{i>\kr}\lambda_i^2}{\big(\lambda+\sum_{i>\kr}\lambda_i\big)^2}\Bigg\}.
\end{align*}
\end{theorem}

\begin{proof}[Proof of Theorem~\ref{theorem:precondition_ridge_fit}]
    Recall that for a given conditioning matrix $\Mb$, the preconditioned ridge regression problem amounts to finding the optimal solution to the following least square problem:
\begin{align*}
 {\wb}_{\mathrm{ridge}}(N, \Mb;\lambda) =   \arg\min_{\wb} \|\Xb\wb - \yb\|_2^2 + \lambda\|\wb\|^2_{{\Mb}}.
\end{align*}
We define $\hat{\wb} = \Mb^{1/2} \wb$ and $\hat{\Xb} =  \Mb^{-1/2}\Xb$. Substitute these into the precondition ridge, we can see that the ridge regression problem is equivalent to solving the following the least square problem:
\begin{align*}
 \arg\min_{\hat{\wb}} \|{\hat{\Xb}\hat{\wb}} - \yb\|_2^2 + \lambda\|{\hat{\wb}}\|_{2}^2,
\end{align*}
which is of the same form as the standard ridge regression. Therefore, we can directly extend the results of ridge regression from Theorem~\ref{thm:lowerbound_ridge} with slight adaption, i.e., taking $\hat{\Hb} = \Hb^{1/2}\Mb^{-1}\Hb^{1/2}$ and $\hat{\wb}^*=\Mb^{1/2} \wb^*$. Then, we can arrive at the risk bound for precondition ridge regression in Theorem~\ref{theorem:precondition_ridge_fit}.
\end{proof}

\section{Analysis of precondition SGD with standard ridge regression}

In this appendix, we present a proof for Theorem~\ref{theorem:HI}. Recall that our proposed design of preconditioning is as follows:
\begin{align*}
    \Gb =   ( \beta \Hb + \Ib \big )^{-1}.
\end{align*}

Then, by Theorem~\ref{theorem:precondition_sgd_fit2}, the learning dynamic of the new problem can be characterized by the matrix:
\begin{align*}
    \tH = \Gb^{1/2} \Hb \Gb^{1/2}.
\end{align*}
In the following section, we denote $\lambda_1,\dots,\lambda_d$ as the sorted eigenvalue for matrix $\Hb$ with corresponding eigenvector $\vb_1,\dots,\vb_d$. 

Before diving into the proof for Theorem~\ref{theorem:HI}, we first prove a set of useful lemmas that we are going to use in the proof. 

The following lemma characterizes the effect of preconditioning on the signal strength of the problem. In particular, it shows that preconditioning does not change the total signal strength.
\begin{lemma}\label{lemma:double_effect}
    For all choice of $\Gb$, we have $$\|\tilde{\wb}^*\big\|_{\tilde{\Hb}}^2 = \|\wb^*\big\|_{\Hb}^2.$$
\end{lemma}
\begin{proof}
    By definition, we have that
     $\norm{\tw^*}^2_{\tH} :=  {\tw}^{*\top} \tH \tw^* $. Substitute in the definition of $\tw^*$ and $\tH$, we have that,
     \begin{align*}
         \norm{\tw^*}^2_{\tH} &:=  {\tw}^{*\top} \tH \tw^* \\
         & =  (\Gb^{-1/2} \wb^*)^\top \Gb^{1/2} \Hb \Gb^{1/2} (\Gb^{-1/2} \wb^*) \\
         & =   \wb^{*\top} (\Gb^{-1/2})^\top \Gb^{1/2} \Hb \Gb^{1/2} (\Gb^{-1/2} \wb^*) \\
         & =   \wb^{*\top} (\Gb^{-1/2}) \Gb^{1/2} \Hb \Gb^{1/2} (\Gb^{-1/2} \wb^*) \\
         & =  \wb^{*\top} \Hb \wb^* \\
         & =   \|\wb^*\big\|_{\Hb}^2
     \end{align*}
\end{proof}

Next, we characterize the updated eigenvalue of the new transformed data covariance $\tH$. In particular, the following lemma shows that the $\tH$ and $\Hb$ share the same eigenvector basis with tractable eigenvalue.
\begin{lemma}\label{lemma:tilde_eigen}
    Suppose $\vb_i$ is the eigenvector of $\Hb$ with eigenvalue $\lambda_i$, 
    then we have that $\vb_i$ is also an eigenvector for $\tH$ with the eigenvalue $\tlambda_i$ with following expression: 
    $$\tilde{\lambda}_i =  \frac{ {\lambda}_i}{\beta {\lambda}_i + 1}.$$
\end{lemma}
\begin{proof}
    Let $\vb_i$ be an arbitrary eigenvector of ${\Hb}$ with eigenvalue ${\lambda}_i > 0$. Then, it is an eigenvector for $\beta {\Hb} + \Ib$ with eigenvalue $\beta {\lambda}_i + 1 > 0$. 
    Therefore, it is also an eigenvector for $\Gb =  (\beta \Hb + \Ib)^{-1}$ with eigenvalue $\frac{1}{\beta {\lambda}_i + \Ib}$. 

    Since $\Gb$ is PSD matrix, $\Gb^{1/2}$ and $\Gb$ have the same eigenvectors. Therefore we have that $\Gb^{1/2}\Hb\Gb^{1/2}$, $\Gb\Hb$ and $\Hb$ have the same eigenvectors. Therefore, $\vb_i$ is also an eigenvector of $\tH = \Gb^{1/2}\Hb\Gb^{1/2}$ with eigenvalues 
    $$\Gb^{1/2}\Hb\Gb^{1/2} = \Gb\Hb[i] = \frac{ \lambda_i}{\beta \lambda_i + 1} $$
\end{proof}

    
The lemma above shows that $\Hb$ and $\tH$ have the same eigenbasis. It also provides a mapping between the eigenvalue of $\Hb$ and $\tH$. Next, we show that the mapping between the eigenvalue of $\Hb$ and $\tH$ is monotonic and therefore does not change the order of eigenvectors and eigenvalues pair, i.e., if $\vb_i$ is the i-th leading eigenvector of $\Hb$ with eigenvalue $\lambda_i$ then it is also the i-th leading eigenvector of $\tH$ with eigenvalue $\tlambda_i$.
\begin{lemma}\label{lemma:monotone}
    Let $$f(a) = \frac{a}{\beta a+1}$$ for some $\beta > 0$. Then we have that $f(a)$ is a monotonic increasing function for $a >0$, i.e., $f(a_1) \geq f(a_2) \iff a_1 \geq a_2$.
\end{lemma}
\begin{proof}
    By taking the derivative of $f(a)$ with respect to $a$, we get that 
    \begin{align*}
        \frac{\partial f(a)}{\partial a} &= \frac{\beta a + 1 - \beta a}{(\beta a +1)^2}\\
        &= \frac{1}{(\beta a +1)^2} \\
        &> 0
    \end{align*}
    Therefore, $f(a)$ is an increasing function with respect to $a$. 
\end{proof}

\begin{proof}[Proof of Theorem~\ref{theorem:HI}]
By Theorem~\ref{theorem:precondition_sgd_fit2}, we know the excessive risk of preconditioned SGD is given by the following, 
\begin{equation}\label{eq:sgd_risk_HI}
    \begin{split}
                 \mathrm{SGDRisk}
&\lesssim \underbrace{\frac{1}{\eta^2 N^2}\cdot\big\|\exp(-N\eta \tilde{\Hb})\tilde{\wb}^*\big\|_{\tilde{\Hb}_{0:k_1}^{-1}}^2 + \|\tilde{\wb}^*\big\|_{\tilde{\Hb}_{k_1:\infty}}^2}_{\mathrm{SGDBias}} \\
&\qquad + \underbrace{\big(\sigma^2+ \|\tw^*\|_{\tH}\big) \cdot \bigg(\frac{k_2}{N}+N\eta^2\sum_{i>k_2}\tilde{\lambda}_i^2\bigg)}_{\mathrm{SGDVariance}}.
    \end{split}
\end{equation}
where the parameter $k_1,k_2\in[d]$ can be arbitrarily chosen, $\tilde{\Hb} = \Gb^{1/2}\Hb\Gb^{1/2}$ and $\tw^* = \Gb^{-1/2}\wb^*$. Let $\tlambda_i$ denote the $i$-th eigenvalue of $\tH$. By Lemma~\ref{lemma:tilde_eigen}, we can obtain the following eigenspectrum of $\tilde{\Hb}$:
$$\tilde{\lambda}_i = \frac{\lambda_i}{\beta \lambda_i + 1}, \quad i \in [d].$$

Then recall the lower   of the risk achieved by ridge regression with parameter $\lambda$:
\begin{equation}\label{eq:ridge_risk_appendix}
    \begin{split}
        \mathrm{RidgeRisk }
\gtrsim \underbrace{\frac{\hat{\lambda}^2}{ N^2}\cdot\big\|\wb^*\big\|_{\Hb_{0:k^*}^{-1}}^2 + \|\wb^*\big\|_{\Hb_{k^*:\infty}}^2}_{\mathrm{RidgeBias }} 
 + \underbrace{\sigma^2\cdot\bigg(\frac{k^*}{N}+\frac{N}{\hat{\lambda}^2}\sum_{i>k^*}\lambda_i^2\bigg)}_{\mathrm{RidgeVariance }},
    \end{split}
\end{equation}
where $\hlambda = \lambda + \sum_{i>k^*}\lambda_i$ and $k^* = \min \{k: N \lambda_k \leq \hlambda\}$.

For the following analysis, we set $k_1,k_2 = k^*$ for the excessive risk of SGD and divide the analysis into bias and variance.

{\bf Bias.} By (\ref{eq:sgd_risk_HI}) The bias   of preconditioned SGD is given as follows,

\begin{align*}
    \mathrm{SGDBias}
&\lesssim \frac{1}{\eta^2 N^2}\cdot\big\|\exp(-N\eta \tilde{\Hb})\tilde{\wb}^*\big\|_{\tilde{\Hb}_{0:k^*}^{-1}}^2 + \|\tilde{\wb}^*\big\|_{\tilde{\Hb}_{k^*:\infty}}^2 \\
\end{align*}
From the equation above, we can observe that the bias of SGD can be decomposed into two intervals: 1) $i \leq k^*$ and 2) $i > k^*$.

We start with the second interval. For $i > k^*$, by Lemma~\ref{lemma:double_effect}, we have that,
\begin{align*}
        \mathrm{SGDBias }[k^*:\infty] & = \|\tilde{\wb}^*\big\|_{\tilde{\Hb}_{k^*:\infty}}^2 \notag \\
        & = \|\wb^*\big\|_{\Hb_{k^*:\infty}}^2 \\
        & = \mathrm{RidgeBias }[k^*:\infty].
\end{align*}
For $i \leq k^*$, because of Lemma~\ref{lemma:monotone}, the order of the eigenvalue and eigenvectors are preserved and we can decompose each term of bias as follows, 
\begin{align}
    \mathrm{SGDBias [i]} &= (\tilde{\wb}^*[i])^2\frac{1}{N^2 \eta ^2\tilde{\lambda}_i}\exp \bigg(-2\eta N\tilde{\lambda}_i\bigg)  \label{eq:sgd_risk_ _term}\\
    \shortintertext{subsitute $\tilde{\lambda}_i$ and $\tilde{\wb}^*[i]$ in (\ref{eq:sgd_risk_ _term}), we can obtain :} 
     &= (\wb^*[i])^2\frac{1}{N^2}\frac{1}{\lambda_i} \bigg (\frac{\beta\lambda_i + 1}{ \eta} \bigg )^2\exp \bigg(-2\eta N \frac{ \lambda_i}{\beta \lambda_i + 1}\bigg)  \notag\\
     &= (\wb^*[i])^2\frac{1}{N^2}\frac{1}{\lambda_i} \hat{\lambda}^2 \frac{1}{\hat{\lambda}^2} \bigg (\frac{\beta\lambda_i + 1}{ \eta} \bigg )^2\exp \bigg(-2\eta N \frac{ \lambda_i}{\beta \lambda_i + 1}\bigg)  \notag\\
     &= \mathrm{RidgeBias}[i]  \bigg (\frac{\beta\lambda_i + 1}{ \eta \hat{\lambda}} \exp \bigg(-\eta N \frac{ \lambda_i}{\beta \lambda_i + 1}\bigg)\bigg )^2  \label{eq:sgd_bias2} 
\end{align}


By  substituting the choice of $\beta = 1/\lambda_{k^*}$ in (\ref{eq:sgd_bias2}), we can obtain that
\begin{align}
     \mathrm{SGDBias [i]} &= \mathrm{RidgeBias }[i]  \bigg (\frac{(\lambda_i/\lambda_{k^*} + 1)}{ \eta \hlambda} \exp \bigg(-  \frac{ N \lambda_i}{ \lambda_i/\lambda_{k^*} + 1} \cdot \eta \bigg)\bigg )^2  \label{eq:sgd_bias3} 
  \end{align}
It is straightforward to verify by taking the derivative with respect to $\lambda_i$ that 
\begin{align*}
    \frac{(\lambda_i/\lambda_{k^*} + 1)}{ \eta \hlambda} \exp \bigg(-  \frac{ N \lambda_i}{ \lambda_i/\lambda_{k^*} + 1} \cdot \eta \bigg)
\end{align*}
monotonically decreasing for $\lambda_i \in (0,1)$. Therefore, substituting this fact into (\ref{eq:sgd_bias3}), we can obtain that 
\begin{align*}
      \mathrm{SGDBias [i]}  & = \mathrm{RidgeBias }[i]  \bigg (\frac{(\lambda_i/\lambda_{k^*} + 1)}{ \eta \hlambda} \exp \bigg(-  \frac{ N \lambda_i}{ \lambda_{k_i}/\lambda_{k^*} + 1} \cdot \eta \bigg)\bigg )^2\\
      &\leq \mathrm{RidgeBias }[i]  \bigg (\frac{(\lambda_{k^*}/\lambda_{k^*} + 1) }{  \eta 
 \hat{\lambda}} \exp \bigg(-  \frac{ N \lambda_{k^*}}{ 2} \cdot \eta \bigg)\bigg )^2 \\
      & \lesssim  \mathrm{RidgeBias }[i] \bigg (\frac{2}{  \eta \hlambda } \exp \bigg(-  \frac{ \eta  N \lambda_{k^*}}{ 2} \bigg)\bigg )^2 
\end{align*}
Now, we can divide the analysis into two cases: {\bf Case I:} $\hlambda \geq \tr(\tH)$ and {\bf Case II:}  $\hlambda \leq \tr(\tH)$.

For {\bf Case I:}, we can pick $\eta = 1/\hlambda$ and obtain that that,
\begin{align*}
        \mathrm{SGDBias [i]} &\leq  \mathrm{RidgeBias }[i] \bigg (\frac{2}{  \eta  \hlambda } \exp \bigg(-  \frac{ \eta  N \lambda_{k^*}}{ 2} \bigg)\bigg )^2  \\
        & \leq \mathrm{RidgeBias }[i] \bigg (\frac{2 \hlambda}{  \hlambda} \exp \bigg(-  \frac{ N \lambda_{k^*}}{ 2 \hlambda} \bigg)\bigg )^2\\
     & \lesssim \mathrm{RidgeBias }[i] \notag
\end{align*}


Therefore, combining the results of the two intervals above, we have that 
$$\mathrm{SGDBias } \lesssim \mathrm{RidgeBiasBoud}$$

For {\bf Case II:}, we can pick $\eta = 1/\tr(\tH)$ and obtain that that, 
\begin{align*}
        \mathrm{SGDBias [i]} &\lesssim  \mathrm{RidgeBias }[i] \bigg (\frac{2\tr(\tH)}{  \hlambda } \exp \bigg(-  \frac{ N \lambda_{k^*}}{ 2 \tr(\tH)} \bigg)\bigg )^2 \\
   \shortintertext{By condition $N \geq 2 \log(\frac{2\tr(\tH)}{\hlambda}) \frac{\tr(\tH)}{\lambda_{k^*}}$, we can obtain} 
     & \lesssim \mathrm{RidgeBias }[i] \notag
\end{align*}
Next, let's consider variance. First recall that variance of preconditioned SGD is given by,
\begin{align}
        \mathrm{SGDVariance } & = (\sigma^2+ \|\tilde{\wb}\|_{\tilde{\Hb}}^2)\cdot\bigg(\frac{k^*}{N}+N\eta^2\sum_{i>k^*}\tilde{\lambda}_i^2\bigg) \notag \\    
         & = (1+\frac{\|\tilde{\wb}\|_{\tilde{\Hb}}^2 }{ \sigma^2})\cdot \sigma^2 \bigg(\frac{k^*}{N}+N\eta^2\sum_{i>k^*}\tilde{\lambda}_i^2\bigg) \notag 
\end{align}
By Lemma~\ref{lemma:double_effect}, we have that $\|\tilde{\wb}\|_{\tilde{\Hb}}^2 = \|{\wb}\|_{{\Hb}}^2$. Substituting this fact into the equation above, we can obtain
 \begin{align*}
  \mathrm{SGDVariance } & =  (1+\frac{\|{\wb}\|_{{\Hb}}^2 }{ \sigma^2})\cdot \sigma^2 \bigg(\frac{k^*}{N}+N \eta^2\sum_{i>k^*}\tilde{\lambda}_i^2\bigg)
\end{align*}

Similar to the bias analysis, we can divide the analysis into two cases: {\bf Case I:} $\hlambda \geq \tr(\tH)$ and {\bf Case II:}  $\hlambda \leq \tr(\tH)$.

For {\bf Case I:} $\hlambda \geq \tr(\tH)$, we pick $\eta = 1/\hlambda$ as for the bias:, 
 \begin{align*}
  \mathrm{SGDVariance } & =  (1+\frac{\|{\wb}\|_{{\Hb}}^2 }{ \sigma^2})\cdot \sigma^2 \bigg(\frac{k^*}{N}+\frac{N}{\hlambda^2}\sum_{i>k^*}\tilde{\lambda}_i^2\bigg)
    \shortintertext{substitute the premise that $\frac{\|{\wb}\|_{{\Hb}}^2 }{ \sigma^2} = \Theta(1)$ and the fact that $\beta \lambda_i + 1 \geq 1$: }
    & \lesssim \Theta(1)\cdot \sigma^2 \bigg(\frac{k^*}{N}+\frac{N}{\hat{\lambda}^2}\sum_{i>k^*}\lambda_i^2\bigg)\\
          & = \Theta(1)\cdot \mathrm{RidgeVariance }\\
          & \lesssim  \mathrm{RidgeVariance }
\end{align*}

For {\bf Case II:} $\hlambda \leq \tr(\tH)$:, we can pick $\eta = 1/\tr(\tH)$ as for the bias and obtain that
\begin{align*}
 \mathrm{SGDVariance } & = (1+\frac{\|{\wb}\|_{{\Hb}}^2 }{ \sigma^2})\cdot \sigma^2 \bigg(\frac{k^*}{N}+\frac{N}{\tr(\hH)^2}\sum_{i>k^*}\tilde{\lambda}_i^2\bigg) \\
  &\leq  (1+\frac{\|{\wb}\|_{{\Hb}}^2 }{ \sigma^2})\cdot \sigma^2 \bigg(\frac{k^*}{N}+\frac{N}{\hlambda^2}\sum_{i>k^*}\tilde{\lambda}_i^2\bigg) \\
    \shortintertext{Similarly, by Lemma~\ref{lemma:double_effect}, the premise and the fact that $\beta \lambda_i + 1 \geq 1$: }
    & \lesssim \Theta(1)\cdot \sigma^2 \bigg(\frac{k^*}{N}+\frac{N}{\hat{\lambda}^2}\sum_{i>k^*}\lambda_i^2\bigg)\\
          & = \Theta(1)\cdot \mathrm{RidgeVariance }\\
          & \lesssim  \mathrm{RidgeVariance }
\end{align*}

Therefore, we have that 
$$ \mathrm{SGDVariance } \lesssim \mathrm{RidgeVariance }$$

Combining all the result above, we have that there exists an $\eta$ and $\beta$ such that 
$$ \mathrm{SGDRisk } \lesssim \mathrm{RidgeRisk }.$$
This completes the proof.

\end{proof}

\section{Analysis of preconditioned SGD with preconditioned ridge regression}
In the previous appendix, we have proved that SGD with preconditioning can indeed match the performance of the standard ridge regression. In this section, we compare the excessive risk of preconditioned SGD and preconditioned ridge. We aim to show that for a given precondition matrix $\Mb$, our proposed design of preconditioned matrix $\Gb$ for SGD with slight modification can remain comparable to the excessive risk of precondition ridge. 

First, recall that by Theorem~\ref{theorem:precondition_ridge_fit}, the key quantities that characterize the excessive risk of preconditioned ridge are given by,
$$\hw = \Mb^{1/2}\wb,$$
and
$$\hH = \Hb^{1/2}\Mb^{-1}\Hb^{1/2}.$$

 Then, recall that the proposed preconditioned matrix for SGD in this case is given by,
\begin{align*}
    \Gb = \frac{1}{\Mb} (\beta \hH + \Ib)^{-1}.
\end{align*}
with transformed covariance matrix,
\begin{align*}
    \tH = \Gb^{1/2}\Hb\Gb^{1/2}.
\end{align*}
Then, recall that $\lambda_1,\dots,\lambda_d$ are the sorted eigenvalue of $\Hb$ with respect to eigenvector $\vb_1,\dots,\vb_d$. Then, we denote $\tlambda_1,\dots,\tlambda_d$ and $\hlambda_1,\dots,\hlambda_d$ the sorted eigenvalues for $\tH$ and $\hH$ respectively. In addition, we consider a representative family of preconditioning matrices for ridge regress that share the same eigenspectrum as the data covariance matrix. We denote $\gamma_i$ to be scaling factors of $\Mb$ on $\lambda_i$. In other words, we have,
\begin{align*}
    \hlambda_i = \frac{\lambda_i}{\gamma_i}.
\end{align*}
Furthermore, we assume that $\Mb$ does not change the order of the eigenvalue for $\Hb$. In other words, we have,
\begin{align*}
    \hlambda_i \leq \hlambda_{i+1}
\end{align*}

Next, we prove a set of useful lemma.  The following lemma is similar to Lemma~\ref{lemma:double_effect}, showing that the preconditioning matrix does not affect the overall signal strength in the excessive risk of preconditioned ridge.
\begin{lemma}\label{lemma:double_effect2}
    For all choice of $\Mb$, we have $$\|\hat{\wb}^*\big\|_{\hat{\Hb}}^2 = \|\wb^*\big\|_{\Hb}^2.$$
\end{lemma}
\begin{proof}
     By definition, we have that
     $\norm{\hw^*}^2_{\hH} :=  {\hw}^{*\top} \hH \hw^* $. Substitute in the definition of $\hw^*$ and $\hH$, we have that,
     \begin{align*}
         \norm{\hw^*}^2_{\hH} &:=  {\hw}^{*\top} \hH \hw^* \\
         & =  (\Mb^{1/2} \wb^*)^\top \Hb^{1/2}\Mb^{-1}\Hb^{1/2} (\Mb^{1/2} \wb^*) \\
         & =  (\Mb^{1/2} \wb^*)^\top \Hb^{1/2}\Mb^{-1/2}\Mb^{-1/2}\Hb^{1/2} (\Mb^{1/2} \wb^*) \\
        \shortintertext{by assumptions of $\Mb$, we can obtain :}
         & =  \wb^{*\top}(\Mb^{1/2})^\top \Mb^{-1/2}\Hb\Mb^{-1/2} (\Mb^{1/2} \wb^*) \\
         & = \wb^{*\top} \Hb \wb^* \\
         & =  \norm{\wb^*}^2_{\Hb}
     \end{align*}
\end{proof}

Next, we characterize the updated eigenvalue of the new transformed data covariance $\hH$ and $\tH$. In particular, the following lemma shows that the $\tH$, $\hH$ and $\Hb$ share the same eigenvector basis with tractable eigenvalue.
\begin{lemma}\label{lemma:tilde_eigen2}
 Suppose $\vb_i$ is the eigenvector of $\Hb$ with eigenvalue $\lambda_i$, 
    then we have that $\vb_i$ is also an eigenvector for $\tH$ and $\hH$ with the eigenvalue $\tlambda_i$ and $\hlambda_i$ with following expression: 
    $$\tilde{\lambda}_i =  \frac{ {\hlambda}_i}{\beta {\hlambda}_i + 1}, \quad \hlambda_i = \frac{\lambda_i}{\gamma_i} .$$
\end{lemma}
\begin{proof}
    Let $\vb_i$ be an arbitrary eigenvector of $\hat{\Hb}$ with eigenvalue $\hat{\lambda}_i > 0$. Then, it is an eigenvector for $\beta \hat{\Hb} + \Ib$ with eigenvalue $\beta \hat{\lambda}_i + 1 > 0$. 
    Therefore, it is also an eigenvector for $\Gb =  \frac{1}{\Mb}(\beta \Hb + \Ib)^{-1}$ with eigenvalue $\frac{1}{\gamma_i}\frac{1}{\beta \hat{\lambda} + \Ib}$. 

    Since $\Gb$ is PSD matrix, $\Gb^{1/2}$ and $\Gb$ have the same eigenvectors. Therefore we have that $\Gb^{1/2}\Hb\Gb^{1/2}$, $\Gb\Hb$ and $\Hb$ have the same eigenvectors. Furthermore, we have that the $i$-th eigenvalues of  $\Gb^{1/2}\Hb\Gb^{1/2} $ is 
    $$\Gb^{1/2}\Hb\Gb^{1/2} = \Gb\Hb[i] = \frac{ \hat{\lambda}_i}{\beta \hat{\lambda}_i + 1} $$
\end{proof}


\begin{proof}[Proof of Theorem~\ref{theorem:HM}]
The proof of this theorem proceeds similarly as the one in Theorem~\ref{theorem:HI}. We first highlight the key different quantities here. 

By Theorem~\ref{theorem:precondition_sgd_fit2}, the learning dynamic of SGD after preconditioning can be characterized by $\tilde{\Hb} = \Gb^{1/2}\Hb\Gb^{1/2}$. By Lemma~\ref{lemma:tilde_eigen2}, we can obtain the following eigenspectrum of $\tilde{\Hb}$ in this case:
$$\tilde{\lambda}_i = \frac{\hlambda_i}{\beta \hlambda_i + 1}, \quad i \in [d],$$
where $\hlambda_i$ is the i-th eigenvalue of the covariance matrix $\hH = \Hb^{1/2}\Mb^{-1}\Hb^{1/2}$.


By Theorem~\ref{theorem:precondition_ridge_fit}, we have that the lower bound of the risk achieved by ridge regression with parameter $\lambda$ and precondition matrix $\Mb$:
\begin{align}
    \mathrm{RidgeRisk} \gtrsim 
 & \underbrace{\frac{\hat{\lambda}}{ N^2}\cdot\big\|\hat{\wb}^*\big\|_{\hat{\Hb}_{0:k^*}^{-1}}^2 + \|\hat{\wb}^*\big\|_{\hat{\Hb}_{k^*:\infty}}^2}_{\mathrm{RidgeBias}}  + \underbrace{\sigma^2\cdot\bigg(\frac{k^*}{N}+\frac{N}{\hat{\lambda}^2}\sum_{i>k^*}\hat{\lambda}_i^2\bigg)}_{\mathrm{RidgeVariance}}.
\end{align}
where $\hw^* = \Mb^{1/2}\wb^*$, $\hat{\lambda} = \lambda + \sum_{i>k^*}\hlambda_i$ and $k^* = \min \{k: \hlambda_k \leq \frac{\hat{\lambda}}{N}\}$. 

Again, let's set $k_1,k_2 = k^*$ and divide the analysis into bias and variance.

{\bf Bias.} Similarly, the bias bound of preconditioned SGD is given as follows,
\begin{align*}
    \mathrm{SGDBias}
&\lesssim \frac{1}{\eta^2 N^2}\cdot\big\|\exp(-N\eta \tilde{\Hb})\tilde{\wb}^*\big\|_{\tilde{\Hb}_{0:k^*}^{-1}}^2 + \|\tilde{\wb}^*\big\|_{\tilde{\Hb}_{k^*:\infty}}^2 \\
\end{align*}
We decompose the bias of SGD into two intervals: 1) $i \leq k^*$ and 2) $i > k^*$.
We start with the second interval. For $i > k^*$, by Lemma~\ref{lemma:double_effect}, we have that,
\begin{align*}
        \mathrm{SGDBias}[k^*:\infty] & = \|\tilde{\wb}^*\|_{\tilde{\Hb}_{k^*:\infty}}^2 \notag \\
        & = \|\wb^*\big\|_{\Hb_{k^*:\infty}}^2 \\
\end{align*}
Similarly, by Lemma~\ref{lemma:double_effect2}, we have that 
\begin{align*}
        \mathrm{RidgeBias}[k^*:\infty] & = \|\hw^*\|_{{\hH}_{k^*:\infty}}^2 \notag \\
        & = \|\wb^*\big\|_{\Hb_{k^*:\infty}}^2 \\
\end{align*}
Therefore, we have $$  \mathrm{SGDBias}[k^*:\infty] = \mathrm{RidgeBias}[k^*:\infty].$$

For $i \leq k^*$, by Lemma~\ref{lemma:monotone} and the premise that $\Mb$ does not alter the order of eigenvalue, the order of $\tlambda_i$ is the same as $\lambda_i$. We can decompose each term of bias bound as follows, 
\begin{align}
    \mathrm{SGDBias[i]} &= (\tilde{\wb}^*[i])^2\frac{1}{N^2 \eta ^2\tilde{\lambda}_i}\exp \bigg(-2\eta N\tilde{\lambda}_i\bigg)  \label{eq:sgd_risk_bound_term}\\
    \shortintertext{subsitute $\tilde{\lambda}_i$ and $\tilde{\wb}^*[i]$ in (\ref{eq:sgd_risk_bound_term}), we can obtain :} 
     &= (\wb^*[i])^2 \gamma_i \frac{1}{N^2}\frac{1}{\lambda_i/\gamma_i} \bigg (\frac{\beta\hlambda_i + 1}{\eta} \bigg )^2\exp \bigg(-2\eta N\frac{\lambda_i}{\gamma_i( \beta \hlambda_i + 1)}\bigg) \notag \\
     &= (\hat{\wb}^*[i])^2\frac{1}{N^2}\frac{1}{{\hlambda_i}} \hat{\lambda}^2 \frac{1}{\hat{\lambda}^2} \bigg (\frac{\beta\hlambda_i + 1}{\eta} \bigg )^2\exp \bigg(-2\eta N\frac{\hlambda_i}{( \beta \hlambda_i + 1)}\bigg)  \notag\\
     &= \mathrm{RidgeBias}[i]  \bigg (\frac{\beta\hlambda_i + 1}{\eta \hat{\lambda}} \exp \bigg(-2\eta N\frac{\hlambda_i}{( \beta \hlambda_i + 1)}\bigg )^2  \label{eq:presgd_bias1}
\end{align}

 By  substituting the choice of $\beta = 1/\hlambda_{k^*}$ in (\ref{eq:presgd_bias1}), we can obtain that
\begin{align}
     \mathrm{SGDBias[i]} &= \mathrm{RidgeBias}[i]  \bigg (\frac{(\hlambda_i/\hlambda_{k^*} + 1)}{ \eta \hlambda} \exp \bigg(-  \frac{ N \hlambda_i}{ \hlambda_i/\hlambda_{k^*} + 1} \cdot \eta  \bigg)\bigg )^2  \label{eq:presgd_bias2} 
  \end{align}
Similarly, we can check by taking derivative with respect to $\hlambda_i$ that,
\begin{align*}
    (\frac{(\hlambda_i/\hlambda_{k^*} + 1)}{ \eta \hlambda} \exp \bigg(-  \frac{ N \hlambda_i}{ \hlambda_i/\hlambda_{k^*} + 1} \cdot \eta  \bigg),
\end{align*}
is monotonically decreasing for $\hlambda_i \in (0,1)$. 
Substitute this fact  into (\ref{eq:presgd_bias2}), we can obtain,
\begin{align*}
      \mathrm{SGDBias[i]}  & = \mathrm{RidgeBias}[i]  \bigg (\frac{(\hlambda_i/\hlambda_{k^*} + 1)}{ \eta \hlambda} \exp \bigg(-  \frac{ N \hlambda_i}{ \hlambda_i/\hlambda_{k^*} + 1} \cdot \eta  \bigg)\bigg )^2 \\
    & \leq \mathrm{RidgeBias}[i]  \bigg (\frac{(\hlambda_{k^*}/\hlambda_{k^*} + 1)}{ \eta \hlambda} \exp \bigg(-  \frac{ N \hlambda_i}{ \hlambda_{k^*}/\hlambda_{k^*} + 1} \cdot \eta  \bigg)\bigg )^2 \\
      & \lesssim  \mathrm{RidgeBias}[i] \bigg (\frac{1}{  \eta \hlambda} \exp \bigg(-  \frac{ \eta   N \hlambda_{k^*}}{ 2} \bigg)\bigg )^2 
\end{align*}
Now, we can divide the analysis into two cases: {\bf Case I:} $\hlambda \geq \tr(\tH)$ and {\bf Case II:}  $\hlambda \leq \tr(\tH)$.

For {\bf Case I:}, we can pick $\eta = 1/\hlambda$ and obtain that that,
\begin{align*}
        \mathrm{SGDBias[i]} &\leq  \mathrm{RidgeBias}[i] \bigg (\frac{2}{  \eta  \hlambda } \exp \bigg(-  \frac{ \eta  N \hlambda_{k^*}}{ 2} \bigg)\bigg )^2  \\
        & \leq \mathrm{RidgeBias}[i] \bigg (\frac{2 \hlambda}{  \hlambda } \exp \bigg(-  \frac{  N \hlambda_{k^*}}{ 2 \hlambda} \bigg)\bigg )^2\\
     & \lesssim \mathrm{RidgeBias}[i] \notag
\end{align*}

Therefore, combining the results of the two intervals above, we have that 
$$\mathrm{SGDBias} \lesssim \mathrm{RidgeBias}$$

For {\bf Case II:}, we can pick $\eta = 1/\tr(\tH)$ and obtain that that, 
\begin{align*}
        \mathrm{SGDBias[i]} &\leq  \mathrm{RidgeBias}[i] \bigg (\frac{\tr(\tH)}{  \hlambda } \exp \bigg(-  \frac{ N \hlambda_{k^*}}{ 2 \tr(\tH)} \bigg)\bigg )^2 \\
   \shortintertext{By condition $N \geq 2 \log(\frac{2\tr(\tH)}{\hlambda}) \frac{\tr(\tH)}{\hlambda_{k^*}}$, we can obtain} 
     & \lesssim \mathrm{RidgeBias}[i] \notag
\end{align*}
Next, let's consider variance.
\begin{align}
        \mathrm{SGDVariance} & = (\sigma^2+ \|\tilde{\wb}\|_{\tilde{\Hb}}^2)\cdot\bigg(\frac{k^*}{N}+N\eta^2\sum_{i>k^*}\tilde{\lambda}_i^2\bigg) \notag \\    
         & = (1+\frac{\|\tilde{\wb}\|_{\tilde{\Hb}}^2 }{ \sigma^2})\cdot \sigma^2 \bigg(\frac{k^*}{N}+N\eta^2\sum_{i>k^*}\tilde{\lambda}_i^2\bigg) \notag 
\end{align}

Again, by Lemma~\ref{lemma:double_effect}, we have that $\|\tilde{\wb}\|_{\tilde{\Hb}}^2 = \|{\wb}\|_{{\Hb}}^2$. Substituting this fact into the equation above, we can obtain
 \begin{align*}
  \mathrm{SGDVariance } & =  (1+\frac{\|{\wb}\|_{{\Hb}}^2 }{ \sigma^2})\cdot \sigma^2 \bigg(\frac{k^*}{N}+N \eta^2\sum_{i>k^*}\tilde{\lambda}_i^2\bigg)
\end{align*}

Similar to the bias analysis, we can divide the analysis into two cases: {\bf Case I:} $\hlambda \geq \tr(\tH)$ and {\bf Case II:}  $\hlambda \leq \tr(\tH)$.

For {\bf Case I:} $\hlambda \geq \tr(\tH)$, we pick $\eta = \hlambda$ as for the bias:, 
\begin{align*}
 \mathrm{SGDVariance} & =  (1+\frac{\|{\wb}\|_{{\Hb}}^2 }{ \sigma^2})\cdot \sigma^2 \bigg(\frac{k^*}{N}+\frac{N}{\hlambda^2}\sum_{i>k^*}\tilde{\lambda}_i^2\bigg) \\
    \shortintertext{substitute the premise that $\frac{\|{\wb}\|_{{\Hb}}^2 }{ \sigma^2} = \Theta(1)$ and the fact that $\beta \hlambda_i + 1 \geq 1$: }
    & \lesssim \Theta(1)\cdot \sigma^2 \bigg(\frac{k^*}{N}+\frac{N}{\hat{\lambda}^2}\sum_{i>k^*}\hlambda_i^2\bigg)\\
          & = \Theta(1)\cdot \mathrm{RidgeVariance}\\
          & \lesssim  \mathrm{RidgeVariance}
\end{align*}

For {\bf Case II:} $\hlambda \leq \tr(\tH)$:, we can pick $\eta = 1/\tr(\tH)$ as for the bias and obtain that
\begin{align*}
 \mathrm{SGDVariance} & = (1+\frac{\|{\wb}\|_{{\Hb}}^2 }{ \sigma^2})\cdot \sigma^2 \bigg(\frac{k^*}{N}+\frac{N}{\tr(\hH)^2}\sum_{i>k^*}\tilde{\lambda}_i^2\bigg) \\
  &\leq  (1+\frac{\|{\wb}\|_{{\Hb}}^2 }{ \sigma^2})\cdot \sigma^2 \bigg(\frac{k^*}{N}+\frac{N}{\hlambda^2}\sum_{i>k^*}\tilde{\lambda}_i^2\bigg) \\
          & \lesssim  \mathrm{RidgeVariance}
\end{align*}

Therefore, we have that 
$$ \mathrm{SGDVariance} \lesssim \mathrm{RidgeVariance}$$

Combining all the results above, we have that there exists an $\eta$ such that 
$$ \mathrm{SGDRisk} \lesssim \mathrm{RidgeRisk}.$$

\end{proof}

 

\section{Analysis of precondition SGD with empirical estimation}
In the previous appendix, we have shown that SGD with preconditioning can consistently match ridge regression with and without preconditioning. In the previous proof, we assume access to the precise information of the data covariance matrix $\Hb$.

In this section, we consider the practical scenarios in which we do not have $\Hb$ but are allowed to access a set of unlabelled data $\{\tilde{x}_i\}_{i=1}^M$, which is sampled from the same distribution as the training data points. We consider the following estimation of the preconditioning matrix $\Gb$, 
\begin{align*}
\tilde \Gb = (\beta\tilde\Xb\tilde\Xb^\top+\Ib)^{-1}.
\end{align*}
where $\tilde{\Xb} = [\tilde{\xb}_1,...,\tilde{\xb}_M] \in \mathbb{R}^{d \times M}$. Then the transformed data covariance matrix is given by
\begin{align*}
\tilde \Hb = \tilde\Gb^{1/2}\Hb\tilde\Gb^{1/2}.
\end{align*}
To conduct a similar analysis as in the theoretical setting, we will then need to characterize the eigenvalues of $\tilde \Hb$. We denote $\mu_1,\dots,\mu_d$ as the eigenvalues for $\tH$ in descending order with corresponding eigenvector $\tv_1,\dots,\tv_d$. We start with proving a set of useful lemmas that are going to be used in the proof.


The following lemma provides the lower bound of the $k$-th largest eigenvalue of $\tilde \Hb$ for any $k=o(M)$.
\begin{lemma}\label{lemma:mu_lower}
Let $\tilde \Hb = \tilde\Gb^{1/2}\Hb\tilde\Gb^{1/2}$ and $\mu_1,\dots,\mu_d$ be its sorted eigenvalues in descending order, we have for any $k=o(M)$, with probability at least $1-\exp(-\Omega(M))$,  it holds that
\begin{align*}
\mu_k \ge \frac{1}{8\beta M + \lambda_k^{-1}}.
\end{align*}
\end{lemma}

\begin{proof}
First, it is easy to see that the matrix $\tilde \Hb$ exhibits the same eigenvalues as $\Kb = \Hb^{1/2}\tilde\Gb\Hb^{1/2}$. Then we will resort to studying the eigenvalues of $\Kb$. In particular, we can get
\begin{align*}
\Kb &= \Hb^{1/2}\tilde\Gb\Hb^{1/2} \\
&= \Hb^{1/2}(\beta\tilde\Xb\tilde\Xb^\top+\Ib)^{-1}\Hb^{1/2} \\
&=(\beta\Hb^{-1/2}\tilde\Xb\tilde\Xb^\top\Hb^{-1/2}+\Hb^{-1})^{-1}.
\end{align*}
Therefore, proving the lower bound of the $k$-th largest eigenvalue of $\Kb$ is equivalent to proving the upper bound of the $k$-th smallest eigenvalue of $\Kb^{-1}$, i.e., $\beta\Hb^{-1/2}\tilde\Xb\tilde\Xb^\top\Hb^{-1/2}+\Hb^{-1}$. First, note that $\Zb:=\Hb^{-1/2}\tilde\Xb$ is the whitened version of $\tilde \Xb$, which transforms $\Xb$ back to the standard Gaussian random matrix with $\Zb\sim \mathbb N(0,1)$. Then in the following, we will focus on studying the eigenvalues of $\beta\Zb\Zb^\top + \Hb^{-1}$.

In particular, the upper bound of the $k$-th smallest eigenvalue of $\Kb^{-1}$ can be proved based on the following identity:
\begin{align*}
\mu_{d-k+1}(\Kb^{-1}) = \inf_{\cS\in\RR^d, \mathrm{rank}(\cS)=k}\max_{\vb\in\cS,\|\vb\|_2=1}\vb^\top\Kb^{-1}\vb.
\end{align*}
The above identity basically implies that the $k$-th smallest eigenvalue of a matrix can be formulated as finding a rank-$k$ subspace in which the largest eigenvalue is minimized. Therefore, a feasible upper bound on $\mu_{d-k+1}(\Kb^{-1})$ can be obtained by finding a rank-$k$ subspace and then calculating the largest eigenvalue therein. In particular, we consider the subspace spanned by the top-$k$ eigenvectors of $\Hb$, denoted by $\cS_k = \mathrm{span}\{\vb_1,\dots,\vb_k\}$. Then, for any unit vector $\ub\in\cS_k$, we have
\begin{align*}
\ub^\top\Kb^{-1}\ub \le \beta \ub^\top\Zb\Zb^\top\ub + \ub^\top\Hb^{-1}\ub\le  \beta \ub^\top\Zb\Zb^\top\ub +\lambda_k^{-1},
\end{align*}
where the last inequality holds since the largest eigenvalue of $\Hb^{-1}$ in the space $\cS_k$ is $\lambda_k^{-1}$. Then, we will seek the upper bound of $ \ub^\top\Zb\Zb^\top\ub$.
Note that all columns of $\Zb$ are i.i.d. random Gaussian vector, therefore we can immediately conclude that for any fixed $\ub$, 
\begin{align*}
\ub^\top\Zb\Zb^\top\ub = \sum_{i=1}^M(\ub^\top\zb_i)^2 \sim \chi^2(M).
\end{align*}
Then by standard tail bound of a Chi-square random variable, we have with probability at least $1-\exp(-Mt^2/8)$, it holds that $\ub^\top\Zb\Zb^\top\ub-M\le tM$. Picking $t=1$, we can obtain 
\begin{align}\label{eq:bound_fixedu}
\PP[\ub^\top\Zb\Zb^\top\ub\ge 2M]\le e^{-M/8}. 
\end{align}

Then  consider an $\epsilon$-net on the space $\cS_k$, denoted as $\cN_\epsilon$, which satisfies $|\cN_\epsilon|= O(1/\epsilon^k)$. Then applying similar proof of Lemma 25 in \cite{bartlett2020benign}, for a sufficiently small $\epsilon$, we can obtain that
\begin{align*}
\max_{\ub\in\cS_k}\ub^\top\Zb\Zb^\top\ub\le (1-c\cdot\epsilon)^{-1}\cdot \max_{\ub\in\cN_\epsilon}\ub^\top\Zb\Zb^\top\ub.
\end{align*}
for some constant $c$.
Now we can pick $\epsilon=1/(4c)$ so that $|\cN_\epsilon|=e^{O(k)}$, then applying the high probability bound for the fixed $\ub$ in \eqref{eq:bound_fixedu} together with union bound over $\cN_\epsilon$, we can obtain that
\begin{align*}
\max_{\ub\in\cS_k}\ub^\top\Zb\Zb^\top\ub\le 4\cdot \max_{\ub\in\cN_\epsilon}\ub^\top\Zb\Zb^\top\ub\le 8M
\end{align*}
with probability at least $1- |\cN_\epsilon|\cdot e^{-M/8} = 1- e^{-\Omega(M)}$, where we use the fact that $k=o(M)$. This immediately implies that with probability at least $1-e^{\Omega(M)}$, we have
\begin{align*}
\mu_{d-k+1}(\Kb^{-1})\le 8\beta n + \lambda_k^{-1}.
\end{align*}
This immediately leads to the lower bound of $\mu_k(\Kb)$, which completes the proof.

\end{proof}

The next lemma provides a relation between $\mu_i$ and $\lambda_i$. In particular, it shows that $\lambda_i$ is a natural upper bound for $\mu_i$.
\begin{lemma}\label{lemma:eigen_value_rel}
    Let $\mu_1,\dots, \mu_d$ be the sorted eigenvalues of
    $\tH = \tilde\Gb^{1/2}\Hb\tilde\Gb^{1/2}$ and let $\lambda_1, \dots, \lambda_d$ be the sorted eigenvalues of $\Hb$.
    Then, we have that $\forall i,$
    \begin{align*}
       \mu_i \leq \lambda_i.
    \end{align*}
\end{lemma}
\begin{proof}
Similarly to the proof of Lemma~\ref{lemma:mu_lower}, we can study the eigenvalues of $\Kb = \Hb^{1/2}\tilde\Gb\Hb^{1/2}$, as it has the same eigenspectrum as $\tilde\Gb^{1/2}\Hb \tilde\Gb^{1/2}$. Similarly, we can obtain,
\begin{align*}
\Kb = \Hb^{1/2}\tilde\Gb\Hb^{1/2} = \Hb^{1/2}(\beta\tilde\Xb\tilde\Xb^\top+\Ib)^{-1}\Hb^{1/2} = (\beta\Hb^{-1/2}\tilde\Xb\tilde\Xb^\top\Hb^{-1/2}+\Hb^{-1})^{-1}.
\end{align*}
Again, let $\Zb:=\Hb^{-1/2}\tilde\Xb$. Then, by Woodbury identity, we have,
\begin{align*}
    \Kb = (\beta \Zb\Zb^{\top} + \Hb^{-1})^{-1}  = \Hb - \Hb\Zb(\Ib+\Zb\Hb\Zb^{\top})^{-1}\Zb^{\top}\Hb  \preceq \Hb
\end{align*}
Therefore, we have that $\Kb \preceq \Hb$ and this means that for $\xb$, we have that 
\begin{align}\label{eq:hk_rel}
    \xb^{\top}\Kb \xb \preceq  \xb^{\top}\Hb \xb 
\end{align}
Courant-Fischer theorem states that the $i$-th eigenvalue of an arbitrary matrix is the minimum Rayleigh quotient of the largest subspace of dimension $i$.
Therefore, by Courant-Fischer theorem, we can write the $i$-th eigenvalue $\mu_i$ of matrix $\Kb$ as,
\begin{align}\label{eq:CF_eign}
    \mu_i = \max_{S:\mathrm{dim}(S) = i} \min_{\xb \in S / \{0\}} \frac{\xb^{\top}\Kb\xb}{\xb^{\top}\xb}. 
\end{align}
Let $S'$ be a subspace of dimension $i$ for which the maximum in (\ref{eq:CF_eign}) is attained. Then, combining (\ref{eq:hk_rel}), we can obtain that
\begin{align*}
     \mu_i &= \min_{\xb \in S' /\{0\}} \frac{\xb^{\top}\Kb\xb}{\xb^{\top}\xb} \\
     &\leq \min_{\xb \in S'/ \{0\}} \frac{\xb^{\top}\Hb\xb}{\xb^{\top}\xb}\\
     &\leq \max_{S:\mathrm{dim}(S) = i} \min_{\xb \in S / \{0\}} \frac{\xb^{\top}\Hb\xb}{\xb^{\top}\xb}\\
     &= \lambda_i.
\end{align*}
This completes the proof.
\end{proof}

The next lemma provides a bound on the trace of the transformed $\tH$ with the estimated $\tG$. 
\begin{lemma}\label{lemma:trace_bound_est}
Suppose  
$$\tilde{\Gb} = (\beta \tilde{\Xb}\tilde{\Xb}^\top + \Ib)^{-1},$$
then, with probability at least $1 - \exp(-\Omega(M))$, we have
\begin{align*}
\tr (\tilde{\Gb}  \mathbf{H}) \leq \frac{1}{1+C}  \tr(\mathbf{H}).
\end{align*}    
for some absolute positive constant $C$.
\end{lemma}

\begin{proof}
    Note that 
    $$\Gb = \left( \beta \tilde{\Xb}\tilde{\Xb}^\top + \Ib \right)^{\dagger}  = \frac{1}{\beta}\left( \tilde{\Xb}\tilde{\Xb}^\top + \frac{1}{\beta}\Ib\right)^{\dagger},$$
    by the Woodbury identity~\citep{guha2013nearly}, we have
\begin{align*}
    \bigg (\tilde{\Xb}\tilde{\Xb}^\top + \frac{1}{\beta}\Ib \bigg )^{\dagger} = \beta \Ib - \beta^2 \tilde{\Xb}(\Ib + \beta \tilde{\Xb}^\top \tilde{\Xb})^{-1}\tilde{\Xb}^\top
\end{align*}
Substitute this back to $\tr (\Gb  \mathbf{H})$ we obtain
\begin{align*}
\tr (\Gb \Hb) & = \frac{1}{\beta} \tr  \left( \left( \tilde{\Xb}\tilde{\Xb}^\top + \frac{1}{\beta}\Ib \right)^{\dagger} \Hb \right)\\
& = \frac{1}{\beta} \tr  \left( \left(  \beta \Ib - \beta^2 \tilde{\Xb}(\Ib + \beta \tilde{\Xb}^\top \tilde{\Xb})^{-1}\tilde{\Xb}^\top \right) \Hb \right)\\
& = \frac{1}{\beta} \bigg ( \tr \left( \beta \Hb - \beta^2 \Hb \tilde{\Xb} (\Ib + \beta \tilde{\Xb}^\top \tilde{\Xb} )^{-1}\tilde{\Xb}^\top \right) \bigg )\\
& =    \tr (\Hb) - \beta \tr  \left( \tilde{\Xb}^\top \Hb \tilde{\Xb} (\Ib + \beta  \tilde{\Xb}^\top \tilde{\Xb} )^{-1} \right). \quad 
\end{align*}
Equivalently, we can rewrite $ \Hb = \sum_i \lambda_i \vb_i \vb_i^\top $, where $\lambda_i$ and $\vb_i$ denote the i-th largest eigenvalue of $\Hb$ and its corresponding eigenvector. Similarly, we can further define $\zb_i = \frac{\tilde{\Xb}^\top \vb_i}{\sqrt{\lambda_i}} \in \mathbb{R}^M$. Then we have the following identities:
\begin{align*}
\tilde{\Xb}^\top \tilde{\Xb} = \sum_i \lambda_i \zb_i \zb_i^\top, \quad \tilde{\Xb}^\top \Hb \tilde{\Xb} = \sum_i \lambda_i^2 \zb_i \zb_i^\top.
\end{align*}
Then, we can rewrite $\tr  \left( \tilde{\Xb}^\top \Hb \tilde{\Xb} (\Ib + \beta \tilde{\Xb}^\top \tilde{\Xb} )^{-1} \right) $ as follows,
\begin{align*}
\tr  \left( \tilde{\Xb}^\top \Hb \tilde{\Xb} (\Ib + \beta \tilde{\Xb}^\top \tilde{\Xb} )^{-1} \right) = \tr  \left( \sum_i \lambda_i^2 \cdot \zb_i^\top \left( \Ib + \beta  \sum_i \lambda_i \zb_i \zb_i^\top \right)^{-1} \zb_i \right)
\end{align*}
Then, we can obtain that, 
\begin{align*}
\tr (\Gb \Hb) & =    \tr (\Hb) - \beta \tr  \left( \tilde{\Xb}^\top \Hb \tilde{\Xb} (\Ib + \beta  \tilde{\Xb}^\top \tilde{\Xb} )^{-1} \right)\\
& =  \sum_i \lambda_i - \beta  \lambda_i^2 \cdot \zb_i^\top \left( \Ib + \beta  \sum_i \lambda_i \zb_i \zb_i^\top \right)^{-1} \zb_i.
\end{align*}
By the Sherman-Morrison formula, we can denote $\Ab= \Ib + \beta  \sum_i \lambda_i \zb_i \zb_i^\top $ and $ \Ab_{-i} = \Ib + \beta  \sum_{j \neq i} \lambda_j \zb_j \zb_j^\top $. Then, we have
\begin{align*}
  \zb_i^\top \left( \Ib + \beta  \sum_i \lambda_i \zb_i \zb_i^\top \right)^{-1} \zb_i &=  \zb_i^\top \left(\Ab_{-i}^{-1} -  \frac{\beta \lambda_i \Ab_{-i}^{-1}\zb_i \zb_i^\top\Ab_{-i}^{-1}}{1+\beta\lambda_i\zb_i^{\top}\Ab_{-i}^{-1}\zb_i} \right) \zb_i \\
  &= \zb_i^\top \Ab_{-i}^{-1}\zb_i -  \frac{\beta \lambda_i \left(\zb_i^{\top}\Ab_{-i}^{-1}\zb_i\right)^2}{1+\beta\lambda_i\zb_i^{\top}\Ab_{-i}^{-1}\zb_i}  \\
  & = \frac{\zb_i^{\top}\Ab_{-i}^{-1}\zb_i}{1+\beta\lambda_i\zb_i^{\top}\Ab_{-i}^{-1}\zb_i} 
\end{align*}
Then, substitute this back, we can obtain
\begin{align*}
\lambda_i - \beta  \lambda_i^2 \cdot \zb_i^\top \left( \Ib + \beta  \sum_i \lambda_i \zb_i \zb_i^\top \right)^{-1} \zb_i &= \lambda_i  -  \beta  \lambda_i^2  \frac{\zb_i^{\top}\Ab_{-i}^{-1}\zb_i}{1+\beta\lambda_i\zb_i^{\top}\Ab_{-i}^{-1}\zb_i}  \\
&=  \frac{\lambda_i}{1+\beta\lambda_i\zb_i^{\top}\Ab_{-i}^{-1}\zb_i}
\end{align*}
Putting all these result back, we have 
\begin{align*}
\tr (\tG \Hb) & = \sum_i \frac{\lambda_i}{1+\beta\lambda_i\zb_i^{\top}\Ab_{-i}^{-1}\zb_i}
\end{align*}

Then, we use the following result adapted from the proof of Lemma 7 in~\citep{tsigler2023benign}.
\begin{lemma}[Lemma 7 in~\citep{tsigler2023benign}]
   Let's denote $$\hat{k} := \min_k \{ k : \lambda_k \frac{1}{\beta} \leq  \frac{1}{\beta} + \sum_{i > k} \lambda_i\},$$ then with probability at least $1 - \exp(\Omega(n))$, it holds that
   $$\zb_i^{\top}\Ab_{-i}^{-1}\zb_i \geq C \cdot \frac{1}{\beta + 
   \beta^2 \sum_{j > \hat{k}}\lambda_j},$$
   where \( C \) is an absolute positive constant.
\end{lemma}
 Based on the above lemma, we can obtain
\begin{align*}
    \tr (\tG \Hb) & = \sum_i \frac{\lambda_i}{1+\beta\lambda_i\zb_i^{\top}\Ab_{-i}^{-1}\zb_i} \\
    &\leq \sum_i \lambda_i/ \left( {1+ C  \beta \lambda_i \cdot \frac{1}{\beta + \beta^2\sum_{j > \hat{k}}\lambda_j}} \right)   
\end{align*}
Next, we decompose the sum above with respect to $\hat{k}$.

For $ i \leq \hat{k} $, by definition of $\hat{k}$, we have that $ \lambda_i \frac{1}{\beta} \geq \frac{1}{\beta} + \sum_{j > \hat{k}}  \lambda_j$. Furthermore, $f(a) = \frac{a}{a+b}$ is a monotonic increasing function with respect to $a$ if $b >0$. Then, for $i \leq \hat{k}$, we can obtain that
\begin{align*}
   \lambda_i/\left( {1+ C  \beta \lambda_i \cdot \frac{1}{\beta + \beta^2\sum_{j > \hat{k}}\lambda_j}} \right) & =  \lambda_i/\left( {1+ \frac{C   \lambda_i}{1 + \beta \sum_{j > \hat{k}}\lambda_j}} \right) \\
       & \leq \lambda_i \left( {1+ \frac{C   \lambda_i}{\lambda_i}} \right)  \\
    & \leq \frac{\lambda_i}{1+C} 
\end{align*}

For $ i > \hat{k}$, we have 
\begin{align}
   \lambda_i/\left( {1+ C  \beta \lambda_i \cdot \frac{1}{\beta + \beta^2\sum_{j > \hat{k}}\lambda_j}} \right) & =  \lambda_i /\left( \frac{1 + \beta \sum_{j > \hat{k}}\lambda_j + C   \lambda_i}{1 + \beta \sum_{j > \hat{k}}\lambda_j}\right)  \notag \\ 
      & =  \frac{\lambda_i (1 + \beta \sum_{j > \hat{k}}\lambda_j)}{1 + \beta \sum_{j > \hat{k}}\lambda_j + C   \lambda_i} \notag
\end{align}
Let $\phi = 1 + \beta \sum_{j > \hat{k}}\lambda_j$. It is easy to verify by taking the derivative with respect to $\phi$ that the expression above is monotonically increasing with respect to $\phi$. Then, substituting the fact that $ \lambda_i  \geq 1 +  \beta \sum_{j > \hat{k}}  \lambda_j$. We can obtain, 
\begin{align*}
      \lambda_i/\left( {1+ C  \beta \lambda_i \cdot \frac{1}{\beta + \beta^2\sum_{j > \hat{k}}\lambda_j}} \right) & \leq \frac{\lambda_i \lambda_i}{\lambda_i + C \lambda_i} \\
      = \frac{\lambda_i}{1+C}
\end{align*}

This completes the proof.

\end{proof}

The following lemma is adapted from Theorem 4.1 in~\citep{loukas2017close}.
\begin{lemma}[Eigenvector of sample covariance, Theorem 4.1 in~\citep{loukas2017close}]\label{lemma:eigen_approx}
    Suppose the sample size for estimating the covariance matrix is $m$, let $\tv_i$ and $\vb_i$ be the eigenvector the sample and actual covariance respectively, for any real number $t >0$, with probability at least $\frac{t}{m}$, we have that
    \begin{align*}
        | \langle \tv_i, \vb_i \rangle | < t.
    \end{align*}
\end{lemma}
The lemma above indicates that for a large enough sample size, with high probability, we have the eigenvector of the sample covariance matrix align with the eigenvector of the actual covariance matrix. For convenience, in the rest of the section, we assume the sample size for estimating the covariance matrix is large enough and we simply use $\vb_i$ for the analysis.

The next lemma characterizes the tail sum of the transformed covariance matrix.

\begin{lemma}\label{lemma:tail_sum_est}
Let  $\tilde \Hb = \tilde\Gb^{1/2}\Hb\tilde\Gb^{1/2}$. Assuming that $\|\wb\|_{\Hb}$ is upper bounded by some constant, then we have that,   
\begin{align*}
  \|\tw\|_{\tH_{k:\infty}} \leq  \|\wb\|_{\Hb_{k:\infty}}.
\end{align*}
\end{lemma}
\begin{proof}
    By definition and Lemma~\ref{lemma:eigen_approx}, we have that 
    \begin{align*}
           \|\tw\|_{\tH_{k:\infty}}  &= \sum_{i > k^*} \mu_i (\vb_i^{\top} \wb)   \\
   \shortintertext{By Lemma~\ref{lemma:eigen_value_rel}, we can obtain:}
         &  \leq   \sum_{i > k^*} \lambda_i (\vb_i^{\top} \wb)  \\
         & = \|\wb\|_{\Hb_{k:\infty}}
    \end{align*}
This completes the proof.
\end{proof}

\begin{proof}[Proof of Theorem~\ref{theorem:HM_est}]
The proof proceeds similarly to the one in Theorem~\ref{theorem:HI}. We set $k_1,k_2 = k^*$ and divide the analysis into bias and variance.

{\bf Bias.} Again, for the bias, we can decompose the bias of SGD into two intervals: 1) $i \leq k^*$ and 2) $i > k^*$.

We start with the second interval. For $i > k^*$, by Lemma~\ref{lemma:tail_sum_est}, we immediately obtain,
\begin{align*}
        \mathrm{SGDBias}[k^*:\infty] & = \|\tilde{\wb}^*\big\|_{\tilde{\Hb}_{k^*:\infty}}^2 \notag \\
        & \leq \|\wb^*\big\|_{\Hb_{k^*:\infty}}^2 \\
        & = \mathrm{RidgeBias}[k^*:\infty].
\end{align*}
For $i \leq k^*$, with Lemma~\ref{lemma:eigen_approx}, we follow the argument in the proof of Theomre~\ref{theorem:HI} with the updated eigenvalues $\mu_i$ and can obtain that,
\begin{align*}
     \mathrm{SGDBias}[i] & = \mathrm{RidgeBias}[i]  \bigg (\frac{1}{\mu_i}\frac{\lambda_i}{\eta \hat{\lambda}} \exp (-\eta N \mu_i) \bigg )^2  
\end{align*}
Again, it is easy to verify by taking the derivative with respect to $\mu_i$ that,
\begin{align*}
    \frac{1}{\mu_i}\frac{\lambda_i}{\eta \hat{\lambda}} \exp (-\eta N \mu_i)
\end{align*}
is monotonically decreasing with respect to $\mu_i$. Therefore, we can obtain that, 
\begin{align}
    \mathrm{SGDBias}[i] & = \mathrm{RidgeBias}[i]  \bigg (\frac{1}{\mu_i}\frac{\lambda_i}{\eta \hat{\lambda}} \exp (-\eta N \mu_i) \bigg )^2  \\
    & \leq \mathrm{RidgeBias}[i]  \bigg (\frac{1}{\eta \hat{\lambda}\mu_{k^*}} \exp (-\eta N \mu_{k^*}) \bigg )^2 \label{eq:bias_est1}
\end{align}
By Lemma~\ref{lemma:mu_lower}, for a estimation sample size of $M$, with high probability, we have,
\begin{align}\label{eq:mu_lower}
    \mu_i \geq \frac{\lambda_i}{8\beta M \lambda_{i} +1} 
\end{align}
Substitute (\ref{eq:mu_lower}) into (\ref{eq:bias_est1}), we obtain,
\begin{align}
    \mathrm{SGDBias}[i] & \leq \mathrm{RidgeBias}[i]  \bigg (\frac{1}{\eta \hat{\lambda}\mu_{k^*}} \exp (-\eta N \mu_{k^*}) \bigg )^2  \notag \\
    &\leq \mathrm{RidgeBias}[i]  \bigg (\frac{8 \beta M \lambda_{k^*}+1}{\eta \hat{\lambda}\lambda_{k^*}} \exp (-\eta N \frac{\lambda_{k^*}}{8\beta M \lambda_{k^*} + 1}) \bigg )^2  \notag
    \shortintertext{Substitute in the choice of $\beta = 1/(8M\lambda_{k^*})$, we obtain}
    &= \mathrm{RidgeBias}[i]  \bigg (\frac{2}{\eta \hat{\lambda}\lambda_{k^*}} \exp (-\eta N \frac{\lambda_{k^*}}{2}) \bigg )^2  
    \label{eq:bias_est2}
\end{align}

Similarly, we consider the analysis in two cases: {\bf Case I:} $\hlambda \geq \tr(\tH)$ and {\bf Case II:}  $\hlambda \leq \tr(\tH)$.

For {\bf Case I:}, we can pick $\eta = 1/\hlambda$ and substitute in (\ref{eq:bias_est2}). We obtain that,
\begin{align*}
        \mathrm{SGDBias}[i] &\leq \mathrm{RidgeBias}[i]  \bigg (\frac{2}{\eta \hat{\lambda}\lambda_{k^*}} \exp (- \eta N \frac{\lambda_{k^*}}{2}) \bigg )^2 \\
        &= \mathrm{RidgeBias}[i]  \bigg (\frac{2\hlambda}{ \hat{\lambda}\lambda_{k^*}} \exp (- N \frac{\lambda_{k^*}}{2\hlambda}) \bigg )^2 \\
        & = \mathrm{RidgeBias}[i] \bigg ( \frac{2 }{ \lambda_{k^*}} \exp\bigg(\frac{-N\lambda_{k^*}}{2\hlambda} \bigg) \bigg )^2\\
        \shortintertext{Provided $N \geq 2 \log (\frac{2}{\lambda_{k^*}})\frac{\hlambda}{\lambda_{k^*}}$ we can obtain}
     & \lesssim \mathrm{RidgeBias}[i] 
\end{align*}

Therefore, combining the results of the two intervals above, we have that 
$$\mathrm{SGDBias} \lesssim \mathrm{RidgeBias}$$

For {\bf Case II:}, we can pick $\eta = 1/\tr(\tH)$ and substitute in (\ref{eq:bias_est2}). We obtain that,
\begin{align*}
        \mathrm{SGDBias}[i] &\leq \mathrm{RidgeBias}[i]  \bigg (\frac{2}{\eta \hat{\lambda}\lambda_{k^*}} \exp (- \eta N \frac{\lambda_{k^*}}{2}) \bigg )^2 \\
        &= \mathrm{RidgeBias}[i]  \bigg (\frac{2\tr(\tH)}{ \hat{\lambda}\lambda_{k^*}} \exp (- N \frac{\lambda_{k^*}}{2\tr(\tH)}) \bigg )^2 \\
        & = \mathrm{RidgeBias}[i] \bigg ( \frac{2 \tr(\tH) }{ \lambda_{k^*} \hlambda} \exp\bigg(\frac{-N\lambda_{k^*}}{2\tr(\tH)} \bigg) \bigg )^2\\
        \shortintertext{Provided $N \geq 2 \log (\frac{2 \tr(\tH)}{\hlambda \lambda_{k^*}})\frac{\tr(\tH)}{\lambda_{k^*}}$ we can obtain}
     & \lesssim \mathrm{RidgeBias}[i] 
\end{align*}

Next, let's consider variance.
\begin{align}
        \mathrm{SGDVariance} & = (\sigma^2+ \|\tilde{\wb}\|_{\tilde{\Hb}}^2)\cdot\bigg(\frac{k^*}{N}+N\eta^2\sum_{i>k^*}{\mu}_i^2\bigg) \notag \\    
         & = (1+\frac{\|\tilde{\wb}\|_{\tilde{\Hb}}^2 }{ \sigma^2})\cdot \sigma^2 \bigg(\frac{k^*}{N}+N\eta^2\sum_{i>k^*}{\mu}_i^2\bigg) \label{eq:var_est1} 
\end{align}
Again, by Lemma~\ref{lemma:double_effect}, we have that 
\begin{align}\label{eq:stn_est}
    \frac{\|\tilde{\wb}\|_{\tilde{\Hb}}^2 }{ \sigma^2} = \frac{\|{\wb}\|_{{\Hb}}^2 }{ \sigma^2} = \Theta(1).
\end{align}
Substitute (\ref{eq:stn_est}) into (\ref{eq:var_est1}), we have,
\begin{align}\label{eq:var_est2}
    \mathrm{SGDVariance} \lesssim \Theta(1) \cdot \sigma^2 \bigg(\frac{k^*}{N}+N\eta^2\sum_{i>k^*}{\mu}_i^2\bigg) 
\end{align}

Similar to the bias analysis, we can divide the analysis into two cases: {\bf Case I:} $\hlambda \geq \tr(\tH)$ and {\bf Case II:}  $\hlambda \leq \tr(\tH)$.

For {\bf Case I:} $\hlambda \geq \tr(\tH)$, we pick $\eta = 1/\hlambda$ as for the bias:, 
\begin{align}\label{eq:var_hatlam}
 \mathrm{SGDVariance} & =  \Theta(1) \cdot \sigma^2 \bigg(\frac{k^*}{N}+\frac{N}{\hlambda^2}\sum_{i>k^*}{\mu}_i^2\bigg) 
 \end{align}
By Lemma~\ref{lemma:eigen_value_rel}, we have that $\mu_i \leq \lambda_i$. Substitute this into into (\ref{eq:var_hatlam}), we can obtain that 
\begin{align*}
     \mathrm{SGDVariance} &\lesssim \Theta(1)\cdot \sigma^2 \bigg(\frac{k^*}{N}+\frac{N}{\hat{\lambda}^2}\sum_{i>k^*}\mu_i^2\bigg)\\
         & \lesssim \Theta(1)\cdot\sigma^2 \bigg(\frac{k^*}{N}+\frac{N}{\hat{\lambda}^2}\sum_{i>k^*}\lambda_i^2\bigg) \\
          & \lesssim  \mathrm{RidgeVariance},
\end{align*}

For {\bf Case II:} $\hlambda \leq \tr(\tH)$, we can pick $\eta = 1/\tr(\tH)$ as for the bias and obtain that
\begin{align*}
 \mathrm{SGDVariance} & = \Theta(1) \cdot \sigma^2 \bigg(\frac{k^*}{N}+\frac{N}{\tr(\hH)^2}\sum_{i>k^*}{\mu}_i^2\bigg) \\
\shortintertext{By premise, we have $\hlambda \leq \tr(\tH)$. Therefore, we can obtain:}
  &\leq  \Theta(1) \cdot \sigma^2 \bigg(\frac{k^*}{N}+\frac{N}{\hlambda^2}\sum_{i>k^*}{\mu}_i^2\bigg) \\
 \shortintertext{Again, following a similar argument in {\bf Case I}, we obtain:}
          & \lesssim  \mathrm{RidgeVariance}
\end{align*}

Therefore, we have that 
$$ \mathrm{SGDVariance} \lesssim \mathrm{RidgeVariance}$$

Combining all the result above, we have that there exists an $\eta$ such that 
$$ \mathrm{SGDRisk} \lesssim \mathrm{RidgeRisk}.$$
\end{proof}

\subsection{Preconditioned SGD vs Preconditioned Ridge}
Next, we consider the case where ridge regression is also given a precondition matrix $\Mb$. First recall that by Theorem~\ref{theorem:precondition_ridge_fit}, the key quantities that characterize the excessive risk of preconditioned ridge are given by,
$$\hw = \Mb^{1/2}\wb,$$
and
$$\mathbf{\Sigma} = \tX\tX^{\top}, \quad \hH = \mathbf{\Sigma}^{1/2}\Mb^{-1}\mathbf{\Sigma}^{1/2}.$$

 Then, recall that the proposed preconditioned matrix for SGD in this case is given by,
\begin{align*}
    \tG = \frac{1}{\Mb} (\beta \hH + \Ib)^{-1}.
\end{align*}
with transformed covariance matrix,
\begin{align*}
    \tH = \tG^{1/2}\Hb\tG^{1/2}.
\end{align*}

First, it is easy to verify by replacing the relevant quantities in Lemma~\ref{lemma:mu_lower} and Lemma~\ref{lemma:eigen_value_rel} with the updated quantities above, we can get the following results.
\begin{lemma}\label{lemma:mu_lower2}
Let $\tilde \Hb = \tilde\Gb^{1/2}\Hb\tilde\Gb^{1/2}$ and $\mu_1,\dots,\mu_d$ be its sorted eigenvalues in descending order, we have for any $k=o(M)$, with probability at least $1-\exp(-\Omega(M))$,  it holds that
\begin{align*}
\mu_k \ge \frac{1}{8\beta M + \hlambda_k^{-1}},
\end{align*}
where $\hlambda_k = \frac{\lambda_k}{\gamma_k}$
\end{lemma}

\begin{lemma}\label{lemma:eigen_value_rel2}
    Let $\mu_1,\dots, \mu_d$ be the sorted eigenvalues of
    $\tH = \tilde\Gb^{1/2}\Hb\tilde\Gb^{1/2}$ and let $\hlambda_1, \dots, \hlambda_d$ be the sorted eigenvalues of $\hH$.
    Then, we have that $\forall i,$
    \begin{align*}
       \mu_i \leq \hlambda_i.
    \end{align*}
\end{lemma}

\begin{proof}[Proof of Theorem~\ref{theorem:HM_est}]
The proof proceeds similarly to the one in Theorem~\ref{theorem:HM}. We set $k_1,k_2 = k^*$ and divide the analysis into bias and variance.

{\bf Bias.} Again, for the bias, we can decompose the bias of SGD into two intervals: 1) $i \leq k^*$ and 2) $i > k^*$.

We start with the second interval. For $i > k^*$, by Lemma~\ref{lemma:tail_sum_est}, we immediately obtain,
\begin{align*}
        \mathrm{SGDBias}[k^*:\infty] & = \|\tilde{\wb}^*\big\|_{\tilde{\Hb}_{k^*:\infty}}^2 \notag \\
        & \leq \|\wb^*\big\|_{\Hb_{k^*:\infty}}^2 \\
        & = \|\hw^*\big\|_{\hH_{k^*:\infty}}^2 \\
        & = \mathrm{RidgeBias}[k^*:\infty].
\end{align*}
where the second last equality is due to Lemma~\ref{lemma:double_effect2}.

For $i \leq k^*$, with Lemma~\ref{lemma:eigen_approx}, we follow the argument in the proof of Theorem~\ref{theorem:HM} with the updated eigenvalues $\mu_i$ and can obtain that,
\begin{align*}
     \mathrm{SGDBias}[i] & = \mathrm{RidgeBias}[i]  \bigg (\frac{1}{\mu_i}\frac{\lambda_i}{\eta \hat{\lambda}} \exp (-\eta N \mu_i) \bigg )^2  
\end{align*}
Again, it is easy to verify by taking the derivative with respect to $\mu_i$ that,
\begin{align*}
    \frac{1}{\mu_i}\frac{\lambda_i}{\eta \hat{\lambda}} \exp (-\eta N \mu_i)
\end{align*}
is monotonically decreasing with respect to $\mu_i$. Therefore, we can obtain that, 
\begin{align}
    \mathrm{SGDBias}[i] & = \mathrm{RidgeBias}[i]  \bigg (\frac{1}{\mu_i}\frac{\lambda_i}{\eta \hat{\lambda}} \exp (-\eta N \mu_i) \bigg )^2 \notag \\
    & \leq \mathrm{RidgeBias}[i]  \bigg (\frac{1}{\eta \hat{\lambda}\mu_{k^*}} \exp (-\eta N \mu_{k^*}) \bigg )^2 \label{eq:bias_estHM1}
\end{align}
By Lemma~\ref{lemma:mu_lower2}, for a estimation sample size of $M$, with high probability, we have,
\begin{align}\label{eq:mu_lowerHM}
    \mu_i \geq \frac{\hlambda_i}{8\beta M \hlambda_{i} +1} 
\end{align}
Substitute (\ref{eq:mu_lowerHM}) into (\ref{eq:bias_estHM1}), we obtain,
\begin{align}
    \mathrm{SGDBias}[i] & \leq \mathrm{RidgeBias}[i]  \bigg (\frac{1}{\eta \hat{\lambda}\mu_{k^*}} \exp (-\eta N \mu_{k^*}) \bigg )^2  \notag \\
    &\leq \mathrm{RidgeBias}[i]  \bigg (\frac{8 \beta M \hlambda_{k^*}+1}{\eta \hat{\lambda}\hlambda_{k^*}} \exp (-\eta N \frac{\hlambda_{k^*}}{8\beta M \lambda_{k^*} + 1}) \bigg )^2  \notag
    \shortintertext{Substitute in the choice of $\beta = 1/(8M\hlambda_{k^*})$, we obtain}
    &= \mathrm{RidgeBias}[i]  \bigg (\frac{2}{\eta \hat{\lambda}\hlambda_{k^*}} \exp (-\eta N \frac{\hlambda_{k^*}}{2}) \bigg )^2  
    \label{eq:bias_estHM2}
\end{align}

Similarly, we consider the analysis in two cases: {\bf Case I:} $\hlambda \geq \tr(\tH)$ and {\bf Case II:}  $\hlambda \leq \tr(\tH)$.

For {\bf Case I:}, we can pick $\eta = 1/\hlambda$ and substitute in (\ref{eq:bias_estHM2}). We obtain that,
\begin{align*}
        \mathrm{SGDBias}[i] &\leq \mathrm{RidgeBias}[i]  \bigg (\frac{2}{\eta \hat{\lambda}\hlambda_{k^*}} \exp (- \eta N \frac{\hlambda_{k^*}}{2}) \bigg )^2 \\
        &= \mathrm{RidgeBias}[i]  \bigg (\frac{2\hlambda}{ \hat{\lambda}\hlambda_{k^*}} \exp (- N \frac{\hlambda_{k^*}}{2\hlambda}) \bigg )^2 \\
        & = \mathrm{RidgeBias}[i] \bigg ( \frac{2 }{ \hlambda_{k^*}} \exp\bigg(\frac{-N\hlambda_{k^*}}{2\hlambda} \bigg) \bigg )^2\\
        \shortintertext{Provided $N \geq 2 \log (\frac{2}{\hlambda_{k^*}})\frac{\hlambda}{\hlambda_{k^*}}$ we can obtain}
     & \lesssim \mathrm{RidgeBias}[i] 
\end{align*}

Therefore, combining the results of the two intervals above, we have that 
$$\mathrm{SGDBias} \lesssim \mathrm{RidgeBias}$$

For {\bf Case II:}, we can pick $\eta = 1/\tr(\tH)$ and substitute in (\ref{eq:bias_estHM2}). We obtain that,
\begin{align*}
        \mathrm{SGDBias}[i] &\leq \mathrm{RidgeBias}[i]  \bigg (\frac{2}{\eta \hat{\lambda}\hlambda_{k^*}} \exp (- \eta N \frac{\hlambda_{k^*}}{2}) \bigg )^2 \\
        &= \mathrm{RidgeBias}[i]  \bigg (\frac{2\tr(\tH)}{ \hat{\lambda}\hlambda_{k^*}} \exp (- N \frac{\hlambda_{k^*}}{2\tr(\tH)}) \bigg )^2 \\
        & = \mathrm{RidgeBias}[i] \bigg ( \frac{2 \tr(\tH) }{ \hlambda_{k^*} \hlambda} \exp\bigg(\frac{-N\hlambda_{k^*}}{2\tr(\tH)} \bigg) \bigg )^2\\
        \shortintertext{Provided $N \geq 2 \log (\frac{2 \tr(\tH)}{\hlambda \hlambda_{k^*}})\frac{\tr(\tH)}{\hlambda_{k^*}}$ we can obtain}
     & \lesssim \mathrm{RidgeBias}[i] 
\end{align*}

Next, let's consider variance.
\begin{align}
        \mathrm{SGDVariance} & = (\sigma^2+ \|\tilde{\wb}\|_{\tilde{\Hb}}^2)\cdot\bigg(\frac{k^*}{N}+N\eta^2\sum_{i>k^*}{\mu}_i^2\bigg) \notag \\    
         & = (1+\frac{\|\tilde{\wb}\|_{\tilde{\Hb}}^2 }{ \sigma^2})\cdot \sigma^2 \bigg(\frac{k^*}{N}+N\eta^2\sum_{i>k^*}{\mu}_i^2\bigg) \label{eq:var_estHM} 
\end{align}
Again, by Lemma~\ref{lemma:double_effect}, we have that 
\begin{align}\label{eq:stn_estHM}
    \frac{\|\tilde{\wb}\|_{\tilde{\Hb}}^2 }{ \sigma^2} = \frac{\|{\wb}\|_{{\Hb}}^2 }{ \sigma^2} = \Theta(1).
\end{align}
Substitute (\ref{eq:stn_estHM}) into (\ref{eq:var_estHM}), we have,
\begin{align}\label{eq:var_estHM2}
    \mathrm{SGDVariance} \lesssim \Theta(1) \cdot \sigma^2 \bigg(\frac{k^*}{N}+N\eta^2\sum_{i>k^*}{\mu}_i^2\bigg) 
\end{align}

Similar to the bias analysis, we can divide the analysis into two cases: {\bf Case I:} $\hlambda \geq \tr(\tH)$ and {\bf Case II:}  $\hlambda \leq \tr(\tH)$.

For {\bf Case I:} $\hlambda \geq \tr(\tH)$, we pick $\eta = 1/\hlambda$ as for the bias:, 
\begin{align}\label{eq:var_hatlamHM}
 \mathrm{SGDVariance} & =  \Theta(1) \cdot \sigma^2 \bigg(\frac{k^*}{N}+\frac{N}{\hlambda^2}\sum_{i>k^*}{\mu}_i^2\bigg) 
 \end{align}
By Lemma~\ref{lemma:eigen_value_rel2}, we have that $\mu_i \leq \hlambda_i$. Substitute this into into (\ref{eq:var_hatlamHM}), we can obtain that 
\begin{align*}
     \mathrm{SGDVariance} &\lesssim \Theta(1)\cdot \sigma^2 \bigg(\frac{k^*}{N}+\frac{N}{\hat{\lambda}^2}\sum_{i>k^*}\mu_i^2\bigg)\\
         & \lesssim \Theta(1)\cdot\sigma^2 \bigg(\frac{k^*}{N}+\frac{N}{\hat{\lambda}^2}\sum_{i>k^*}\hlambda_i^2\bigg) \\
          & \lesssim  \mathrm{RidgeVariance},
\end{align*}

For {\bf Case II:} $\hlambda \leq \tr(\tH)$, we can pick $\eta = 1/\tr(\tH)$ as for the bias and obtain that
\begin{align*}
 \mathrm{SGDVariance} & = \Theta(1) \cdot \sigma^2 \bigg(\frac{k^*}{N}+\frac{N}{\tr(\hH)^2}\sum_{i>k^*}{\mu}_i^2\bigg) \\
\shortintertext{By premise, we have $\hlambda \leq \tr(\tH)$. Therefore, we can obtain:}
  &\leq  \Theta(1) \cdot \sigma^2 \bigg(\frac{k^*}{N}+\frac{N}{\hlambda^2}\sum_{i>k^*}{\mu}_i^2\bigg) \\
 \shortintertext{Again, following a similar argument in {\bf Case I}, we obtain:}
          & \lesssim  \mathrm{RidgeVariance}
\end{align*}

Therefore, we have that 
$$ \mathrm{SGDVariance} \lesssim \mathrm{RidgeVariance}$$

Combining all the result above, we have that there exists an $\eta$ such that 
$$ \mathrm{SGDRisk} \lesssim \mathrm{RidgeRisk}.$$

\end{proof}

\section{Setting Discussion and Extension}\label{appendix:extension}
In this paper, our main objective is to close an open question in the existing learning theory literature. As such, most of the settings follow the previous studies. In this appendix, we present a further discussion and illustration of the chosen setting and assumption. In addition, we discuss how our analysis and results here can be extended beyond the linear model to cases such as Relu regression.
\subsection{Discussion on  Assumption}
{\bf Example.} It has been illustrated in previous studies that distribution with sub-Gaussian property (which is common assumption in many learning theory studies) is enough to satisfy Assumption~\ref{assump:data_distribution},~\ref{assump:fourth_moment}, and ~\ref{assump:model_noise}. In addition, consider the whitened data $$\mathbf{z}:=\mathbf{H}^{-1/2}\mathbf{x}.$$ If $\mathbf{z}$ consists of i.i.d coordinates with bounded first-order moments, then possessing sub-exponential properties would be also sufficient to meet these assumptions.

{\bf Comparison.} Typical analysis of SGD~\citep{jain2018parallelizing} would require a bound on the second moment of the stochastic gradient.
In our context, this translates to a second moment bound on the gradient term given by,
$$(\la \mathbf{w}_t,\mathbf{x}_t \ra -y_t) \mathbf{x}_t,$$
which involves a product of terms  $\mathbf{x}_t$. This automatically implies a bound on the fourth-order moment of $\mathbf{x}$, which aligns well with our assumption. Therefore, our assumption is not stronger than the typical analysis of SGD.

\subsubsection{Memory-efficient Implementation of Tail-average SGD.}
Tail-average SGD has been shown to achieve better bias and is necessary for the excessive risk analysis of SGD. However, the analysis and the proof do not need the tail-average length to be exactly $N/2$. Anything that is in a similar order would work.

To see that tail-average SGD is more than just a theoretical interest, we discuss the efficient implementation for the tail-average SGD in practice. We divide the discussion into two cases: case I: $N$ is known and case II: $N$ is unknown. For case I, f $N$ is known, the tail averaging can be efficiently implemented with a simple moving average. On the other hand, even if $N$  is unknown, efficient heuristics can be designed based on the scenario. For example, if the rough range of $N$ is known, we can design a simple dynamic moving average tracker. To illustrate, suppose $2^r \leq N \leq 2^{r+1}$, let
\begin{align*}
        \mathbf{w}_a &= \frac{1}{2^{r-1}}\sum_{t = 2^{r-1}}^{2^r}\mathbf{w}^{(t)} \\
        \mathbf{w}_b &= \frac{1}{N-2^{r}}\sum_{t = 2^{r}}^{N}\mathbf{w}^{(t)}
\end{align*}
   
Consider the tail average with
\begin{align*}
    \bar{\mathbf{w}} &= \frac{\mathbf{w}_a 2^{r-1} + \mathbf{w}_b (N-2^r)}{2^{r-1}+N-2^r}, \quad \mathrm{if}\quad N-2^{r} < 2^{r-1} \\
    \bar{\mathbf{w}} &= \mathbf{w}_b, \quad \mathrm{otherwise}
\end{align*}

This would guarantee a tail-averaging of the same order of $N/2$. Therefore, our theory can still be applied and we only need to store $\mathbf{w}_a$ and $\mathbf{w}_b$.

\subsection{ReLU Regression}\label{appendix:relu}
Similarly, we denote a feature vector in a separable Hilbert space $\cH$ as $\xb\in\cH$. The dimensionality of $\cH$ is denoted as $d$, which can be infinite-dimensional when applicable. We use $y\in\RR$ to denote a response. Then, the ReLU regression problem instance is to minimize the following objective:
\begin{align*}
& L(\wb^*) = \min_{\wb\in\cH}L(\wb),\\
& L(\wb):= \EE_{(\xb,y)\sim \mathbb{D}}\big[(\mathrm{ReLU}(\xb^\top \wb) - y)\big]^2,
\end{align*}
where  $\mathrm{ReLU}(.) := \max\{., 0\}$ is the Rectified Linear Unit (ReLU). An problem of interest to the learning theory community related to ReLU regression to compare the performance of SGD and the performance of GLM-tron~\citep{kakade2011efficient}. Under ReLU regression, the SGD update is given by,
\begin{align*}
    & \wb_t = \wb_{t-1} - \eta_t \cdot \gb_t, \\
    & \gb_t := (\mathrm{ReLU}(\xb_t^\top \wb_{t-1}) -\yb_t )\xb_t \cdot  \mathbf{1}[\xb_t^\top \wb_{t-1} > 0],
\end{align*}
where $\mathbf{1}[.]$ is the standard indicator function and $\eta_t$ is the learning rate or step size at the $t$-th iteration.  

Furthermore, the GLM-tron update is given by,
\begin{align*}
    \wb_t = \wb_{t-1} - \eta_t \cdot (\mathrm{ReLU}(\xb_t^\top \wb_{t-1}) - y_t)\xb_t.
\end{align*}
Comparing these two algorithms, the only difference is that GLM-tron ignores the derivative of ReLU(·) in its updates.

Excessive risk bound and instance-wise comparison between SGD and GLM-tron in ReLU regression has recently been given in~\citep{wu2023finite}. In particular, they study this problem under a decaying step size schedule $(\eta_t)_t$ given by,
\begin{align}
    \eta_t =  \left\{
    \begin{array}{ll}
          \eta_{t-1}/2, & t \% (N\log N) = 0, \\
          \eta_{t-1}, & \mathrm{otherwise}, \\
    \end{array} 
    \right.
\end{align}
where $N$ is the sample size. Under the ``well-specified'' setting (bounded noise/variance) setting, \cite{wu2023finite} has provided a sharp excess risk characterization for both SGD and GLM-tron for the ReLU regression, which is also dependent on the interplay between the eigenspectrum of the data covariance and the optimal model parameter. 

In addition, \cite{wu2023finite} has conducted an instance-based comparison between SGD and GLM-tron for symmetric Bernoulli data ($\PP(x) = \PP(-x)$). In particular, they have shown that when there is no noise in the data, SGD unavoidably suffers from a constant risk in expectation, while GLM-tron can still obtain a small risk. Under the case with ``well-specified'' noise, GLM-tron can provably achieve (at least) comparable performance to SGD, i.e.,
\begin{align*}
    \text{GLM-tron Risk} \lesssim \text{SGD Risk}.
\end{align*}

The overall result in~\citep{wu2023finite} suggests that GLM-tron might be more effective in ReLU regression. 

In particular, the excess risk upper bound of GLM-tron in ReLU regression is given by,

\begin{align}
    \text{GLM-tronRisk} \lesssim \underbrace{\left\lVert \prod_{t=1}^{N}\left( \Ib -\frac{\eta_t}{2} \Hb \right) \wb^* \right\lVert_{\Hb}^2}_{\text{GLM-tronBias}} + \underbrace{(\sigma^2+ \|\wb^*\|_{\Hb}^2) \left(\frac{\log N k}{N} + \frac{N\eta_0^2}{\log N}\sum_{i > k} \lambda_i^2\right)}_{\text{GLM-tronVariance}}, 
\end{align}
where $k > 0$ is arbitrary. In particular, when we pick 
$$ k = k^* = \max \{ k : \lambda_k \geq \log N/(\eta_0 N)\},$$
We can arrive the excess risk lower bound for GLM-tron
\begin{align}
    \text{GLM-tronRisk} \gtrsim \underbrace{\left\lVert \prod_{t=1}^{N}\left( \Ib -\frac{\eta_t}{2} \Hb \right) \wb^* \right\lVert_{\Hb}^2}_{\text{GLM-tronBias}} + \underbrace{(\sigma^2+ \|\wb^*\|_{\Hb}^2) \left(\frac{\log N k^*}{N} + \frac{N\eta_0^2}{\log N}\sum_{i > k^*} \lambda_i^2\right)}_{\text{GLM-tronVariance}}, 
\end{align}
From the results above, we can observe that the excess risk of GLM-tron is again dependent on the alignment between the eigenspectrum of the covariance matrix and the optimal model parameter. The analysis and proof in ~\citep{wu2023finite} reveals that the iterate dynamic of GLM-tron is closely related to the dynamic of linear regression and therefore, can be analysed by comparing it to linear regression. Therefore, we can naturally extend the analysis in the paper to study the reconditioning of GLM-tron in ReLU regression which is given by,
\begin{align*}
    \wb_t = \wb_{t-1} - \eta_t \cdot (\mathrm{ReLU}(\xb_t^\top \wb_{t-1}) - y_t)\cdot \Gb \xb_t,
\end{align*}
where $\Gb$ is the preconditioning matrix. The iterative dynamic of the learning process of ReLU regression is similar to the one for linear regression. Therefore, through a similar computation as in Theorem~\ref{theorem:precondition_sgd_fit2}, we can see that the excess risk of preconditioned GLM-tron can again be obtained by changing the covariance matrix and optimal model parameter into $\tH = \Gb^{1/2}\Hb \Gb^{1/2}$ and $\tw = \Gb^{-1/2}\wb^*$. This leads to the following excess risk bound for preconditioned GLM-tron,
\begin{align}
    \text{GLM-tronRisk} \lesssim \underbrace{\left\lVert \prod_{t=1}^{N}\left( \Ib -\frac{\eta_t}{2} \tH \right) \tw^* \right\lVert_{\tH}^2}_{\text{GLM-tronBias}} + \underbrace{(\sigma^2+ \|\tw^*\|_{\tH}^2) \left(\frac{\log N k}{N} + \frac{N\eta_0^2}{\log N}\sum_{i > k} \tlambda_i^2\right)}_{\text{GLM-tronVariance}}, 
\end{align}
where $k > 0$ is arbitrary. In particular, when we pick 
$$ k = k^* = \max \{ k : \tlambda_k \geq \log N/(\eta_0 N)\},$$
We can arrive the excess risk lower bound for GLM-tron
\begin{align}
    \text{GLM-tronRisk} \gtrsim \underbrace{\left\lVert \prod_{t=1}^{N}\left( \Ib -\frac{\eta_t}{2} \tH \right) \tw^* \right\lVert_{\tH}^2}_{\text{GLM-tronBias}} + \underbrace{(\sigma^2+ \|\tw^*\|_{\Hb}^2) \left(\frac{\log N k^*}{N} + \frac{N\eta_0^2}{\log N}\sum_{i > k^*} \tlambda_i^2\right)}_{\text{GLM-tronVariance}}, 
\end{align}

 Therefore, based on the above analysis, we can readily extend the result and framework in this paper to study preconditioning for GLM-tron in ReLU regression.  Given the significance of GLM-tron in ReLU regression, it would be interesting to study whether preconditioning can help optimize the excess risk of GLM-tron. This renders an interesting future research direction but beyond the scope of this paper.



\section{Concrete Example for Comparing Ridge and SGD} \label{appendix:concrete_example}
In this appendix, we present concrete examples where preconditioned SGD achieve better performance than ridge regression and where preconditioning closes the performance gap between SGD and ridge regression.
\subsection{Preconditioned SGD outperforms Ridge Regression}
Consider a problem instance with  $\mathbf{H}$'s spectrum such that,
\begin{align}
    \lambda_i =  \left\{
    \begin{array}{ll}
         \frac{\log N}{ N^{1/2}} , & i = 1, \\
          \frac{1 - \log N/ N^{1/2}}{N}, & 1<i\leq N, \\
          0, & \mathrm{Otherwise}
    \end{array} 
    \right.
\end{align}
 
and with the true parameter $\mathbf{w}^*$ defined as:
\begin{align}
    \wb^*[i] =  \left\{
    \begin{array}{ll}
        \sigma \sqrt{\frac{N^{1/2}}{\log N}} , & i = 1, \\
          0, & \mathrm{Otherwise}
    \end{array} 
    \right.
\end{align}
By picking the step size $\eta = N^{-1/2}$, using Theorem~\ref{thm:standard_sgd} one can show the excess risk of SGD is given by,
\begin{align*}
    \mathrm{SGDRisk} \lesssim \frac{\sigma^2}{N}
\end{align*}
Similarly, using Theorem~\ref{thm:lowerbound_ridge}, one can get that the excess risk of the ridge is given by,
\begin{align*}
    \mathrm{RidgeRisk} \gtrsim \frac{\sigma^2}{N^{1/2}\log N}
\end{align*}
Therefore, we have, 
\begin{align*}
    \mathrm{SGDRisk} \lesssim \mathrm{RidgeRisk}
\end{align*}

\subsection{Preconditioning Closing the Gap between SGD and Ridge}
Consider the problem instance  where $\mathbf{H}$ is a diagonal matrix with $\mathbf{H}[1][1] = 1, \mathbf{H}[2][2] = \frac{1}{N \kappa(N)}$ (where $\kappa(N) = \mathrm{tr}(\mathbf{H})/(N \lambda_{ \mathrm{min} (N, d)})$ , and the remaining entries are zero. Set $\mathbf{w}^*[2] = N \kappa(N)$ and all other entries of $\mathbf{w}^*$ to zero. Under this case, it is challenging for SGD to learn $\mathbf{w}^*[2]$ due to the small magnitude of the second eigenvalue. By applying our proposed preconditioning with $\beta = \frac{1}{N \kappa(N)}$, we can show that PreSGD's performance becomes comparable with that of ridge regression.

\section{Comparison to Newton's method}\label{appendix:newton}
In this appendix, we present a more in-depth discussion and comparison between our proposed preconditioning matrix with Newton's method.

First, recall that Newton's method can be viewed as a preconditioning technique with the second derivative (Hession) which corresponds to the covariance matrix in our setting.  On the other hand, recall that our design of $\mathbf{G}$ is given by $$(\beta \mathbf{H} + I)^{-1}.$$ 
As $\beta$ becomes large, $\beta \mathbf{H}$ would become dominant and $\mathbf{G}$ would become more similar to a version of Newton’s method. This adjustment tends to reduce bias and accelerate the convergence rate but may increase the risk of overfitting (larger variance). Conversely, a small $\beta$ makes $\mathbf{G}$ approximate an identity matrix, reducing PreSGD to plain SGD. In this case, PreSGD will have a relatively larger bias due to its insufficient optimization for some small eigenvalues. Therefore, we can achieve a better bias-variance trade-off by adjusting the $\beta$ in $\mathbf{G}$.
\section{Preconditioned Ridge Matching Preconditioned SGD}\label{appendix:ridge_match_sgd}
In the previous results, we have shown that for every preconditioned Ridge, we can tune the preconditioned matrix and learning rate for SGD so that it is guaranteed to have a comparable performance to the preconditioned Ridge. A natural follow-up question is whether the inverse direction is true. In this appendix, we show that the inverse direction is indeed true. In other words, given an instance of preconditioned SGD, we can tune the preconditioned matrix and the regularization parameter of ridge regression so that it is guaranteed to have comparable performance to the preconditioned SGD. Similarly, we consider a representative family of $\Gb$ that has the same eigenbasis and preserve the order of eigenvalue of the covariance matrix. Then, formally, we have the following theorem.

\begin{theorem}[Preconditioned-Ridge comparable to preconditioned SGD] \label{theorem:MG}
Let $\Gb$ be the preconditioned matrix for SGD. Let $N$ be a sufficiently large training sample size for both ridge and SGD.  Then, for any SGD solution that is generalizable and any $\eta < \lambda_1$, there exists a choice of the preconditioned matrix $\Mb^*$ and a regularization parameter $\lambda^*$  such that 
$$  \cE\big[\wb_{\mathrm{ridge}}(N,\Mb^*;\lambda^*)\big] \lesssim  \cE\big[{\wb}_{\mathrm{sgd}}(N,\Gb; \eta)\big] .$$
\end{theorem}

To prove the above theorem, we would need the excess risk lower bound for SGD and the excessive risk upper bound for the ridge. We start by restating the standard excess risk results from the existing literature. According to~\cite{zou2021benign} and~\cite{bartlett2020benign}, we have the following SGD risk upper bound and ridge risk lower bound.

\begin{theorem}[SGD risk lower bound and Ridge risk upper bound]\label{thm:standard_sgd_lower}
    Consider SGD with tail-averaging with initialization $\wb_0=\bm{0}$.
Suppose the stepsize satisfies $\eta \leq 1/\lambda_1$.
Then the excess risk can be lower-bounded as follows,
\begin{align*}
\mathrm{SGDRisk}
&\geq  \sgdbias + \sgdvar,
\end{align*}
where,
\begin{align*}
 \sgdbias & \gtrsim \frac{1}{\eta^2N^2}\cdot\big\|\exp\big(-N\eta \Hb \big)\wb^*\big\|_{\Hb_{0:k^*}^{-1}}^2 + \big\|\wb^*\big\|_{\Hb_{k^*:\infty}}^2, \\
 \sgdvar & \gtrsim \frac{\sigma^2 }{N}\cdot\bigg(k^* + N^2\eta^2 \sum_{i> k^*}\lambda_i^2\bigg) +  \|\wb^*\|_\Hb^2 \frac{\eta}{\lambda_1} \sum_{i > k^{\dagger}} \lambda_i^2,
\end{align*}
where $k^* = \max \{k:\lambda_k \geq 1/(N\eta)\}$, and $k^{\dagger} = \max \{k: \lambda_k \geq 2/(3N\eta)\}$.

 Let $\lambda\ge 0$ be the regularization parameter, $n$ be the training sample size and $\wb_{\mathrm{ridge}}(N;\lambda)$ be the output of ridge regression. 
 Then 
\begin{align*}
\mathrm{RidgeRisk} \leq \ridgebias +  \ridgevar ,
\end{align*}
where,
\begin{align*}
\ridgebias &\lesssim \bigg(\frac{\lambda+\sum_{i>\kr}\lambda_i}{N}\bigg)^2\cdot\|\wb^*\|_{\Hb_{0:\kr}^{-1}}^2+\|\wb^*\|_{\Hb_{\kr:\infty}}^2,\notag\\
\ridgevar &\lesssim\sigma^2\cdot\bigg\{\frac{\kr}{N}+\frac{N\sum_{i>\kr}\lambda_i^2}{\big(\lambda+\sum_{i>\kr}\lambda_i\big)^2}\Bigg\},
\end{align*}
where $ \kr := \min\left\{k: b  \lambda_{k+1} \le  (\lambda+\sum_{i>k}\lambda_i)/n\right\} $.
\end{theorem}

By a similar argument as for the proof for Theorem~\ref{theorem:precondition_ridge_fit} and Theorem~\ref{theorem:precondition_sgd_fit2}, we can immediately obtain the lower bound of excess risk for preconditioned SGD and the upper bound of excessive for preconditioned Ridge.

\begin{theorem}[Excessive risk lower bound of preconditioned SGD]\label{theorem:precondition_sgd_fit_lower}    
Let $\Gb$ be a given preconditioned matrix. Consider SGD update rules with (\ref{eq:pred_sgd}) and (\ref{eq:pred_out}). Suppose Assumptions~\ref{assump:data_distribution}, ~\ref{assump:fourth_moment} and ~\ref{assump:model_noise} hold. Let 
$\tilde{\Hb} = \Gb^{1/2}\Hb\Gb^{1/2}, \quad \text{and} \quad \tilde{\wb}^*=\Gb^{-1/2} \wb^*.$ Suppose the step size $\eta$ satisfies $\eta \leq 1/ \tlambda_1$, where $\tlambda_1,\dots, \tlambda_d$ are the sorted eigenvalues of $\tH$ in descending order. Then SGD has the following excessive risk lower bound:
\begin{align*}
    & \mathrm{SGDRisk} \gtrsim  \sgdbias + \sgdvar, \label{eq:predsgd_risk} \\
 \sgdbias & \gtrsim \frac{1}{\eta^2N^2}\cdot\big\|\exp\big(-N\eta \tH \big)\tw^*\big\|_{\tH_{0:k^*}^{-1}}^2 + \big\|\tw^*\big\|_{\tH_{k^*:\infty}}^2, \\
 \sgdvar & \gtrsim \frac{\sigma^2 }{N}\cdot\bigg(k^* + N^2\eta^2 \sum_{i> k^*}\tlambda_i^2\bigg) +  \|\tw^*\|_{\tH}^2 \frac{\eta}{\tlambda_1} \sum_{i > k^{\dagger}} \tlambda_i^2,
\end{align*}
where $k^* = \max \{k:\tlambda_k \geq 1/(N\eta)\}$, and $k^{\dagger} = \max \{k: \tlambda_k \geq 2/(3N\eta)\}$.
\end{theorem}

\begin{theorem}[Excessive risk upper bound for ridge with preconditioning]\label{theorem:precondition_ridge_fit_upper}    
Consider ridge regression with parameter $\lambda > 0$ and precondition matrix $\Mb$. Suppose Assumptions~\ref{assump:data_distribution}, ~\ref{assump:fourth_moment} and ~\ref{assump:model_noise} hold. Let 
$\hat{\Hb} = \Hb^{1/2}\Mb^{-1}\Hb^{1/2}, \quad \text{and} \quad \hat{\wb}^*=\Mb^{1/2} \wb^*.$
For a constant $\hat{\lambda}$ depending on $\lambda$, $\hat{\Hb}$, $N$ , and $\kr = \min\{k: N\hlambda_k \lesssim \hat{\lambda}\}$, ridge regression has the following excessive risk upper bound:
\begin{align*}
\mathrm{RidgeRisk} \lesssim \ridgebias +  \ridgevar ,
\end{align*}
where,
\begin{align*}
\ridgebias &\lesssim \bigg(\frac{\lambda+\sum_{i>\kr}\hlambda_i}{N}\bigg)^2\cdot\|\hw^*\|_{\hH_{0:\kr}^{-1}}^2+\|\hw^*\|_{\hH_{\kr:\infty}}^2,\notag\\
\ridgevar &\lesssim\sigma^2\cdot\bigg\{\frac{\kr}{N}+\frac{N\sum_{i>\kr}\hlambda_i^2}{\big(\lambda+\sum_{i>\kr}\hlambda_i\big)^2}\Bigg\},
\end{align*}
where $\hlambda_1,\dots, \hlambda_d$ are the sorted eigenvalues for $\hH$ in descending order.
\end{theorem}

Now that we have all the ingredients needed, we can start proving Theorem~\ref{theorem:MG}. The overall proof structure is similar to the previous results. The idea is to first align the excess risk of Ridge and SGD by tuning the regularization parameter $\lambda$. Then, we can construct a preconditioning matrix $\Mb$ based on the excess risk profile of preconditioned SGD for the given $\Gb$.

\begin{proof}[Proof of Theorem~\ref{theorem:MG}]
    Recall that to show that preconditioned ridge is comparable with preconditioned SGD is enough to show that 
    \begin{align*}
        \mathrm{RidgeRisk} \lesssim \mathrm{SGDRisk}.
    \end{align*}

    Then, by Theorem~\ref{theorem:precondition_sgd_fit_lower}, we have the excess risk lower bound of SGD given by,

    \begin{align*}
     \mathrm{SGDRisk} & \gtrsim   \underbrace{\frac{1}{\eta^2N^2}\cdot\big\|\exp\big(-N\eta \tH \big)\tw^*\big\|_{\tH_{0:k^*}^{-1}}^2 + \big\|\tw^*\big\|_{\tH_{k^*:\infty}}^2}_{\mathrm{SGDBias}} \\
     & + \underbrace{\frac{\sigma^2 }{N}\cdot\bigg(k^* + N^2\eta^2 \sum_{i> k^*}\tlambda_i^2\bigg) +  \|\tw^*\|_{\tH}^2 \frac{\eta}{\tlambda_1} \sum_{i > k^{\dagger}} \tlambda_i^2,}_{\mathrm{SGDVariance}}
    \end{align*}
    where $\tlambda_1, \dots \tlambda_d$ are sorted eigenvalues for $\tH$, $\tilde{\Hb} = \Gb^{1/2}\Hb\Gb^{1/2}$,   $\tilde{\wb}^*=\Gb^{-1/2} \wb^*.$, $k^* = \max \{k:\tlambda_k \geq 1/(N\eta)\}$, and $k^{\dagger} = \max \{k: \tlambda_k \geq 2/(3N\eta)\}$.

    Similarly, by Theorem~\ref{theorem:precondition_ridge_fit_upper}, we have the excess risk upper bound of ridge given by,
    \begin{align*}
\mathrm{RidgeRisk}  \lesssim   & \underbrace{\bigg(\frac{\lambda+\sum_{i>\kr}\hlambda_i}{N}\bigg)^2\cdot\|\hw^*\|_{\hH_{0:\kr}^{-1}}^2+\|\hw^*\|_{\hH_{\kr:\infty}}^2}_{\mathrm{RidgeBias}} \\
 & + \underbrace{\sigma^2\cdot\bigg(\frac{\kr}{N}+\frac{N\sum_{i>\kr}\hlambda_i^2}{\big(\lambda+\sum_{i>\kr}\hlambda_i\big)^2}\Bigg)}_{\mathrm{RidgeVariance}},
\end{align*}
where $\hlambda_1,\dots, \hlambda_d$ are the sorted eigenvalues for $\hH$ in descending order, $\hat{\Hb} = \Hb^{1/2}\Mb^{-1}\Hb^{1/2}$, $\hat{\wb}^*=\Mb^{1/2} \wb^*$ and  $ \kr := \min\left\{k: b  \lambda_{k+1} \le  (\lambda+\sum_{i>k}\lambda_i)/n\right\}$.

    Let $\Gb$ be a given preconditioned matrix for SGD and let $g_i$ denote the weight of $\Gb$ on the i-th direction. By assumption, $\Gb$ preserves the orientation and order of eigenvalues of $\Hb$. By Theorem~\ref{theorem:precondition_sgd_fit_lower}, we can wrtie,
    $$\tlambda_i = g_i \lambda_i.$$

    As our goal is just to show the existence of matrix $\Mb$, we can consider a preconditioned $\Mb$ with similar property as $\Gb$ and write, 
    $$\hlambda_i = \frac{\lambda_i}{m_i},$$
    where $m_i$ is the weight of $\Mb$ on the i-th direction. 

    Similar to the previous proof, we can decompose the risk profile into bias and variance and then compare them separately. 
    
  We start with aligning the risk profiles of SGD and ridge by picking the ridge regression regularization 

    $$\lambda^* = \frac{b}{\eta} - \sum_{i > k^*}\hlambda_i.$$
    
    Doing so leads to $\kr = k^*$. 
    
    {\bf Variance.} Then we start by comparing the variance. We denote $\hlambda^* = \lambda^*+\sum_{i>k^*}\hlambda_i$ and have the variance for preconditioned Ridge as,
    \begin{align*}
         \ridgevar & \lesssim \sigma^2\cdot\bigg(\frac{k^*}{N}+\frac{N}{(\hlambda^*)^2}\sum_{i>k^*}\hlambda_i^2\bigg)\\ 
         \intertext{By making $m_i \geq \hlambda^*/(g_i\eta)$ for $i > k^*$, we can get} 
          & \lesssim \frac{\sigma^2 }{N}\cdot\bigg(k^* + N^2\eta^2 \sum_{i> k^*}\tlambda_i^2\bigg) \\
          & \lesssim \frac{\sigma^2 }{N}\cdot\bigg(k^* + N^2\eta^2 \sum_{i> k^*}\tlambda_i^2\bigg) +  \|\tw^*\|_{\tH}^2 \frac{\eta}{\tlambda_1} \sum_{i > k^{\dagger}} \tlambda_i^2, \\
         & = \sgdvar
    \end{align*}
    Next, we focus on the bias. Under same set up as the variance, the Ridge bias is given by,
    \begin{align*}
        \ridgebias \lesssim \bigg(\frac{\hlambda^*}{N}\bigg)^2\cdot\|\hw^*\|_{\hH_{0:k^*}^{-1}}^2+\|\hw^*\|_{\hH_{k^*:\infty}}^2
    \end{align*}
Therefore, we have that 
$$ \mathrm{RidgeVariance} \lesssim \mathrm{SGDVariance}$$
    
{\bf Bias.} Similar to the previous analysis, we can decompose the bias of Ridge into two intervals: 1) $i \leq k^*$ and 2) $i > k^*$.

We start with the second interval. For $i > k^*$, by Lemma~\ref{lemma:double_effect2}, we immediately obtain,
\begin{align*}
        \mathrm{RidgeBias}[k^*:\infty] & = \|\hw^*\big\|_{{\hH}_{k^*:\infty}}^2 \notag \\
        & = \|\wb^*\big\|_{\Hb_{k^*:\infty}}^2 \\
        & = \|\tw^*\big\|_{\tH_{k^*:\infty}}^2 \\
        & = \mathrm{SGDBias}[k^*:\infty].
\end{align*}
where the second last equality is due to Lemma~\ref{lemma:double_effect}.

Therefore, we have $$  \mathrm{SGDBias}[k^*:\infty] = \mathrm{RidgeBias}[k^*:\infty].$$

For $i \leq k^*$, similar to the previous analysis, we decompose each term of bias bound as follows, 
\begin{align}
    \mathrm{SGDBias[i]} &= (\tilde{\wb}^*[i])^2\frac{1}{N^2 \eta ^2\tilde{\lambda}_i}\exp \bigg(-2\eta N\tilde{\lambda}_i\bigg) \\
    \intertext{subsitute $\tilde{\lambda}_i$ and $\tilde{\wb}^*[i]$, we can obtain :} 
     &= (\wb^*[i])^2 \frac{1}{N^2}\frac{1}{\lambda_i g_i^2 \eta^2} \exp (-2\eta N \lambda_i g_i) \notag \\
     \intertext{by making $m_i \leq \exp(-N\eta \lambda_i g_i)\hlambda^*/g_i$, we get}
     &\gtrsim (\wb^*[i])^2\frac{m_i^2}{N^2}\frac{(\hlambda^*)^2}{{\lambda_i}}  \notag\\
     &= (\tw^*[i])^2\frac{1}{N^2}\frac{(\hlambda^*)^2}{{\tlambda_i}} \notag\\
     &= \mathrm{RidgeBias}[i]  \notag
\end{align}
Therefore, we have that 
$$ \mathrm{RidgeBias} \lesssim \mathrm{SGDBias}$$

Combining all the results above, we have that there exists an $\Mb$ and $\lambda$ such that 
$$ \mathrm{RidgeRisk} \lesssim \mathrm{SGDRisk}.$$

\end{proof}

\end{document}